\documentclass[11pt]{article}

\usepackage{rlr}

\newcommand{\Paren}[1]{\left(#1\right)}

\usepackage{boxedminipage}

\newcommand{\lmax}{\lambda_{\max}}
\newcommand{\Norm}[1]{\left\lVert#1\right\rVert}
\newcommand{\normt}[1]{\norm{#1}_2}
\newcommand{\bignorm}[1]{\big\lVert#1\big\rVert}

\DeclareMathOperator{\tr}{tr}

\newcommand{\iprod}[1]{\langle#1\rangle}

\DeclareMathOperator{\OPT}{OPT}

\newcommand{\tO}{\tilde{O}}

\newcommand{\gest}{\text{Gradient Estimation}}

\newcommand{\psimp}[1][\eta]{\Delta_{#1, n}}

\newcommand{\nit}{O  (\cn \log \frac{\norm{w^*}}{\eta \sigma \kappa} )}

\newcommand{\var}[1]{\text{var}{#1}}

\newcommand{\hc}{\tau}
\newcommand{\cn}{\kappa}
\newcommand{\cp}{C_1}
\newcommand{\cetakappa}{C_2}
\newcommand{\cge}{C_3}
\newcommand{\ccov}{C_4}
\newcommand{\cfour}{C_6}
\newcommand{\cff}{C_5}

\newenvironment{myenumerate}
{ \begin{enumerate}
   }
{ \end{enumerate}                  } 

\newcommand{\cD}{\mathcal D}

\title{Optimal Robust Linear Regression in Nearly Linear Time}
\author{Yeshwanth Cherapanamjeri\thanks{EECS, University of California Berkeley. \texttt{yeshwanth@berkeley.edu}} \and Efe Aras\thanks{EECS, University of California Berkeley. \texttt{efearas96@berkeley.edu}} \and Nilesh Tripuraneni\thanks{EECS, University of California Berkeley. \texttt{nilesh\_tripuraneni@berkeley.edu}} \and Michael I. Jordan\thanks{EECS, University of California Berkeley. \texttt{jordan@cs.berkeley.edu}} \and Nicolas Flammarion\thanks{School of Computer and Communication Sciences, EPFL. \texttt{nicolas.flammarion@epfl.ch}} \and Peter L. Bartlett\thanks{EECS, University of California Berkeley. \texttt{peter@berkeley.edu}}}

\begin{document}

\maketitle

\begin{abstract}
    We study the problem of high-dimensional robust linear regression where a learner is given access to $n$ samples from the generative model $Y = \inp{X}{w^*} + \epsilon$ (with $X \in \mathbb{R}^d$ and $\epsilon$ independent), in which an $\eta$ fraction of the samples have been adversarially corrupted. We propose estimators for this problem under two settings: (i) $X$ is L4-L2 hypercontractive, $\mb{E} [XX^\top]$ has bounded condition number and $\epsilon$ has bounded variance and (ii) $X$ is sub-Gaussian with identity second moment and $\epsilon$ is sub-Gaussian. In both settings, our estimators:
\begin{enumerate}
    \item Achieve optimal sample complexities and recovery guarantees up to log factors and
    \item Run in near linear time ($\tilde{O}(nd / \eta^6)$).
\end{enumerate}
Prior to our work, \emph{polynomial} time algorithms achieving near optimal sample complexities were only known in the setting where $X$ is Gaussian with identity covariance and $\epsilon$ is Gaussian, and \emph{no} linear time estimators were known for robust linear regression in any setting. Our estimators and their analysis leverage recent developments in the construction of faster algorithms for robust mean estimation to improve runtimes, and refined concentration of measure arguments alongside Gaussian rounding techniques to improve statistical sample complexities.

\end{abstract}

\section{Introduction}
\label[section]{introduction}

Least-squares regression is amongst the oldest and most fundamental statistical methods -- its use dates back over two centuries to the seminal works of \citet{gauss} and \citet{legendre} who used it to estimate the trajectories of celestial bodies. Since then, least-squares regression has found numerous applications in varied fields like finance~\citep{dielman2001applied}, epidemiology~\citep{bhaskaran2013time}, astronomy~\citep{isobe1990linear} and biostatistics~\citep{mcdonald2009handbook}.
However, the algorithmic primitives underlying these estimators assume that the data being used to perform such analysis is clean and well-curated. This assumption is simply not true in modern datasets which often contain extremely noisy and sometimes even adversarial data.

In line with such considerations, we study the problem of high-dimensional robust linear regression. In the standard linear regression setup, one has access to $n$ i.i.d.~samples from the following generative model $Y = \inp{X}{w^*} + \epsilon$ where $X$ and $\epsilon$ are independent and $\epsilon$ is mean $0$ with $\text{Var}(\epsilon) \leq \sigma^2$. In the robust setting, an adversary is allowed to observe the generated samples and arbitrarily perturb an $\eta$ fraction of them. The goal is to estimate the true regression vector, $w^*$, given such corrupted samples. We consider two concrete scenarios:
\vspace*{-.15cm}
\begin{enumerate}
    \item \textbf{Heavy Tailed Case:} $X$ satisfies an L4-L2 hyper-contractivity assumption and
    \vspace*{-.15cm}
    \item \textbf{Sub-Gaussian Case:} $X$ and $\epsilon$ are sub-Gaussian with $\mb{E} [X X^\top] = I$.
\end{enumerate}
\vspace*{-.15cm}

In both settings, our proposed estimators run in near linear-time with near optimal sample complexity and achieve information theoretically optimal recovery guarantees of $O(\sigma \sqrt{\eta})$ in the heavy tailed case assuming a bound on the condition number of $\mb{E}[XX^\top]$ and $O(\sigma \eta \log 1 / \eta)$ in the sub-Gaussian case which is information theoretically optimal up to a factor $O(\sqrt{\log 1 / \eta})$. In contrast, all previous approaches to robust linear regression suffer from either sub-optimal sample complexities, slow running times or both~\citep{klivans2018efficient,prasad2018robust,diakonikolas2019sever,diakonikolas2019efficient}. 

Our algorithm is based on the iterative robust gradient estimation framework from \citep{diakonikolas2019sever} and \citep{prasad2018robust}. In both these works, a candidate parameter vector is maintained and a robust estimate of the gradient of the function being optimized is obtained by running a robust estimation procedure on a collection of contaminated gradient estimates. For example, the gradients at some current estimate $w$ would consist of the vectors $G_i(w)  = (\inp{X_i}{w} - Y_i)X_i$ for linear regression. The robust gradient is then used to iteratively improve the estimate $w$. Since, the number of iterations is typically small, the computational cost of the algorithm is dominated by the cost of a single iteration. To obtain our improved computational performance, we leverage recent developments in algorithms for robust mean estimation~\citep{cheng2019high,DBLP:conf/nips/DongH019} to implement the robust estimation step in near-linear time.

To ensure the success of the algorithm, one needs to show that the robust estimation procedure returns a good estimate of the gradient. In typical approaches to this problem, one is required to exhibit a set of weights, $s_i$, on the gradients such that $\sum s_i G_i(w)$ is close to the true gradient and $\sum s_i G_i(w) G_i(w)^\top$ is spectrally bounded. The usual approach to ensuring this condition is to bound the spectral norm of the random tensor $W = n^{-1} \sum X_i^{\otimes 4}$. However, note that the length of a typical $X_i$ is on the order of $\sqrt{d}$ which means that the spectral norm of $X_i^{\otimes 4}$ is $O(d^2)$. Therefore, one would need $n \approx d^2$ for the spectral norm of $W$ to be small. To circumvent this issue, we instead work with a carefully chosen subset of our data based on our current estimate and use an intricate generalization of Gaussian rounding schemes previously employed in the context of heavy-tailed mean estimation~\citep{lecue2019robust} to exhibit the existence of a suitable set of weights without sacrificing sample complexity.

For the sub-Gaussian case, we generalize results from \citep{diakonikolas2019efficient} for the case where the covariates $X$ and error $\epsilon$ are Gaussian to the milder assumptions of sub-Gaussianity while still achieving the same recovery error. Interestingly, we show that one can obtain this improved error using the heavy tailed algorithm by adding a single correction step at its conclusion to improve the estimate from an error of $O(\sqrt{\eta}\sigma)$ to $O(\eta \log 1/ \eta  \sigma)$. This is also applicable to the robust mean estimation scenario where it conceptually simplifies the presentation of \citep{cheng2019high} and \citep{DBLP:conf/nips/DongH019}. 

\subsection{Related work}

There is long line of work in the statistics community which has led to the construction of estimators robust to the presence of adversarial noise. For example, for the problem of robust mean estimation, \citet{huber1964robust} first proposed an estimator in the one dimensional case which was later generalized to the multi-dimensional case by \citet{tukey1975mathematics}. Since then, the expansion of these ideas has resulted in the development of robust estimators for other tasks like linear regression and covariance estimation~\citep{huber2004robust}. Unfortunately, there are no known polynomial time algorithms to compute several of these estimators. As a remedy to these shortcomings, a recent line of work in the computer science community have devised polynomial time robust estimators achieving near optimal recovery guarantees for a range of statistical estimation problems including mean estimation, linear regression and covariance estimation~\citep{lai2016agnostic,diakonikolas2016robust,diakonikolas2017being,prasad2018robust,diakonikolas2019sever,klivans2018efficient,kothari2018robust,diakonikolas2018robustly,hopkins2018mixture,steinhardt2017resilience,charikar2017learning,diakonikolas2018list,diakonikolas2019recent,diakonikolas2019efficient}. For robust mean estimation, recent work has resulted in algorithms running in near linear time and achieving optimal sample complexity and recovery error~\citep{cheng2019high,DBLP:conf/nips/DongH019}.

The problem of robust linear regression has been previously considered in the works of \citet{klivans2018efficient,diakonikolas2019efficient,prasad2018robust,diakonikolas2019sever,balakrishnan2017computationally,bhatia2015robust,bhatia2017consistent,suggala2019adaptive}. Of these, \citep{bhatia2015robust,bhatia2017consistent,suggala2019adaptive} only tolerate adversarial corruption in the responses $Y_i$ and do not apply in our scenario where the covariates may also be adversarially manipulated. From the viewpoint of computational complexity, none of the aforementioned works have algorithms with run times sub-quadratic in dimension. In terms of sample complexity, all previous algorithms are sub-optimal requiring the number of samples to be at least quadratic in the dimension barring the work of \citet{diakonikolas2019efficient} which applies only in the setting where $X$ and $\epsilon$ are normally distributed. 

Our work is most closely related to \citep{diakonikolas2019sever,prasad2018robust} which consider a general set-up where one aims to optimize a loss function, say $f(w) = \mb{E} [(\inp{X_i}{w} - Y_i)^2]$ for linear regression given access to stochastic gradients from $f$, a fraction of which have been adversarially corrupted. However, while their algorithms may be augmented with faster robust gradient estimation solvers to speed up runtime, their sample complexities are sub-optimal with sample complexities quintic and quadratic in dimension respectively~\citep{diakonikolas2019sever,prasad2018robust} and achieve optimal recovery error only in the heavy tailed setting. \citet{diakonikolas2019efficient} achieve near optimal recovery error and sample complexity under the assumption that both $X$ and $\epsilon$ are Gaussian with $X$ having identity covariance matrix but tolerate arbitrary covariance matrices by exploiting a robust covariance estimation procedure for Gaussian inputs at the cost of sub-optimal sample complexity.

\paragraph{Recent Work:} Two contemporary works \citep{bakshi2020robust,zhu2020robust} also study the problem of robust linear regression in greater generality than in this paper and in particular, \citet{bakshi2020robust} obtain nuanced information theoretically optimal recovery guarantees which scale with the number of available moments of the distribution. However, algorithmically exploiting these assumptions is computationally and statistically expensive. Consequently, both these approaches suffer from impractically large runtime ($O(n^{O(k)})$) and sample complexity ($O(d^{O(k)})$)\footnote{Assuming $k$ moments of the distribution are available. In our settings, $k \geq 4$.}. In addition, we also do not require the more restrictive assumption of certifiability on the moments of the distribution.

\subsection{Notation}
We often use $\mc{D}$ to refer to a distribution over $(X, Y)$, $\mc{G} = \{G_i\}_{i = 1}^n$ and $\mc{Z} = \{Z_i\}_{i = 1}^N$ to refer to sets of points and $\mc{H}$ and its variants to refer to events in a probability space. We use $\Delta_\delta$ to refer to the subset of the probability simplex defined by $\{s \in \mb{R}^n: \sum_{i = 1}^n s_i = 1, 0 \leq s_i \leq \frac{1}{(1 - \delta)n}\}$ and $\mc{E}_\delta$ to refer to the extreme points of $\Delta_\delta$; that is $\mc{E}_\delta = \{s\in \mb{R}^n: \sum_{i = 1}^n s_i = 1 ,s_i \in \{\frac{1}{(1 - \delta)n}, 0\}\}$. Given a set of points $\{Z_i\}_{i=1}^n$, a subset $S \subseteq [n]$ and $s \in \Delta_\delta$, we denote by $\mb{E}_S [f(X_i)] = \abs{S}^{-1} \sum_{i \in S} f(X_i) $ and $\mb{E}_s [f(X_i)] = \sum_{i = 1}^n s_i f(X_i)$. We use $\text{Supp} (s)$ to denote the set of indices for which $s_i > 0$. We also use $G_i(w)=(\langle X_i,w-w^*\rangle-Y_i)X_i$ to refer to the gradient of the linear regression objective evaluated on the $i$th sample at $w$, and $G^*(w)=\E[G_i(w)]$ to refer to the population gradient at $w$. Finally, for a vector $v$ and matrix $M$, we will use $\norm{v}$ and $\norm{M}$ to denote the Euclidean and spectral norms of $v$ and $M$ respectively and $\lmax (M)$ for the largest eigenvalue of $M$.

\section{Main Results}
\label{sec:main_res}


In this section, we formally present the main results of our paper. For a generative model, $\mc{D}$, we formally describe the corruption model below: 

\begin{definition}
    \label{def:etacor}
    Given $\eta > 0$ and a distribution, $\mc{D}$, a set of samples, $\{(X_i, Y_i)\}_{i = 1}^n$, is $\eta$-corrupted if it is generated according to the following model:
    \begin{myenumerate}
        \item $\{(X'_i, Y'_i)\}_{i = 1}^n$ are generated i.i.d.~from the distribution $\mc{D}$
        \item An adversary is allowed to inspect the samples and arbitrarily perturb any $\eta n$ of them.
    \end{myenumerate}
\end{definition}
For the heavy tailed scenario, we make the following assumptions on the generative model $Y = \inp{X}{w^*} + \epsilon$:

\begin{assumption}
\label{as:htassump}
The (uncorrupted) generative model $Y=\langle X, w^{*} \rangle + \epsilon$ satisfies:
\begin{myenumerate}
    \item $X$ is a random vector with second moment matrix $\Sigma$. Furthermore, $X$ is L4--L2 hyper-contractive. That is, for all $\norm{u} = 1$, we have:
\begin{equation*}
    \mb{E} [\inp{u}{X}^4] \leq \hc \cdot \lprp{\mb{E} [\inp{u}{X}^2]}^2
\end{equation*}
for an absolute constant $\hc$. For normalization purposes, we assume $\norm{\Sigma} =  1$. 
    
    \item The condition number of $\Sigma$ is bounded by $\cn < \infty$. 
    \item We have that $X$ satisfies $\norm{X} \leq \cp \sqrt{d}$ almost surely for some absolute constant $\cp > 1$.
    \item The noise variable $\epsilon$ is zero mean and independent of $X$ with variance bounded by $\sigma^2$. 
\end{myenumerate}

Additionally, we assume that $\eta$ satisfies $\eta \leq \frac{\cetakappa}{\kappa^2}$ for a suitably small constant\footnote{This condition is explained in \cref{lem:htgapxtemp}.} $\cetakappa > 0$.
\end{assumption}

The assumptions stated above are standard for heavy tailed linear regression where some condition on the anti-concentration of $XX^\top$ is required to obtain finite sample complexities~\citep{lecue2016performance,oliveira2016lower}. Note that the boundedness assumption $\norm{X} \leq O(\sqrt{d})$ is not a restriction and is made for ease of presentation -- since an unbounded random vector $X$ satisfying  hypercontractivity  can be truncated at level $O(\sqrt{d})$ without significantly distorting the second moment matrix, $\Sigma$ or L4-L2 hypercontraction constant, $\tau$ (See \cref{lem:trunclem} in \cref{ap:htc}). Given a set of random vectors satisfying \cref{as:htassump} save for the boundedness condition in 3., we can simply discard samples with large norms as a preprocessing step so they satisfy \cref{as:htassump} without affecting our analysis or conclusions.

The main result of our paper for the heavy tailed scenario is described in the following theorem:

\begin{theorem}
    \label{thm:htmain}
    Suppose a distribution $\mc{D}$ satisfies \cref{as:htassump}. Then there exists an algorithm which when given $n = \tO (d / \eta)$ $\eta$-corrupted samples from $\mc{D}$, runs in time $\tO (\kappa nd / \eta^6)$ and returns an estimate, $\hat{w}$ satisfying:
    \begin{equation*}
        \norm{\hat{w} - w^*} \leq O(\kappa \sqrt{\eta} \ \sigma),
    \end{equation*}
    with probability at least $0.9$.
\end{theorem}

In the setting where $\kappa = O(1)$, we obtain the information theoretically optimal recovery error of $O(\sqrt{\eta} \sigma)$ with optimal sample complexity up to log factors (\cref{thm:lb_ht}). Note that it is information theoretically \emph{impossible} to achieve parameter recovery independent of $\kappa$ under these assumptions (\cref{cor:condDep}). For the sub-Gaussian scenario, we make the following stronger assumptions on the generative model:

\begin{assumption}
\label{as:sgassump}
The generative model $Y = \inp{X}{w^*} + \epsilon$ satisfies:
\begin{myenumerate}
    \item $X$ is a random vector with second moment matrix $I$. Furthermore, we assume that there exists an absolute constant, $\psi$, such that for every $\norm{u} = 1$:
    \begin{equation*}
        \mb{P} \lprp{\abs{\inp{X}{u} - \mb{E} \inp{X}{u}} \geq t} \leq 2 \exp \lprp{- \frac{t^2}{2\psi^2}}.
    \end{equation*}
    \item The noise variable $\epsilon$ is zero mean, variance $\sigma^2$ independent of $X$ and is sub-gaussian with sub-gaussianity parameter $\phi = O(\sigma)$. That is, $\epsilon$ satisfies:
    \begin{equation*}
        \mb{P} (\abs{\epsilon} > t) \leq 2 \exp \lprp{- \frac{t^2}{2\phi^2}}.
    \end{equation*}
\end{myenumerate}

Additionally, we assume that $\eta$ satisfies $\eta \leq \cetakappa$ for a sufficiently small constant $\cetakappa > 0$.
\end{assumption}

With these stronger assumptions on the generative model, we obtain improved recovery guarantees as detailed in the following theorem:

\begin{theorem}
    \label{thm:sgmain}
    Suppose a distribution $\mc{D}$ satisfies \cref{as:sgassump}. Then, there exists an algorithm which when given $n = \tO (d / \eta^2)$ $\eta$-corrupted samples from $\mc{D}$, runs in time $\tO (nd / \eta^6)$ and returns an estimate, $\hat{w}$, satisfying:
    \begin{equation*}
        \norm{\hat{w} - w^*} \leq O(\eta \log 1 / \eta \ \sigma),
    \end{equation*}
    with probability at least $2 / 3$.
\end{theorem}

We note that in comparison to the heavy tailed scenario, we obtain an improved recovery guarantee of $O(\eta \log 1 / \eta \ \sigma)$ as opposed to the $O(\sqrt{\eta} \ \sigma)$ error obtained in \cref{thm:htmain}. The guarantee obtained in \cref{thm:sgmain} is information theoretically optimal up to a factor of $\sqrt{\log 1 / \eta}$ (\cref{thm:lb_sg}) and the sample complexity is optimal up to log factors \citep{diakonikolas2017being, gao2017robust}.

\paragraph{Remark:} It is possible to improve the runtime dependence on $\eta$ to a logarithmic dependence on $1 / \eta$ by appealing to refined solvers \citep{DBLP:conf/nips/DongH019} for the class of semidefinite programs we use here. However, we do not pursue this avenue in this work.
\section{Algorithm}

In this section, we describe our algorithms which are used in \cref{thm:htmain,thm:sgmain}. In \cref{ssec:htalg}, we describe the algorithm for the heavy tailed scenario and build on this algorithm to construct an algorithm for the sub-Gaussian case in \cref{ssec:sgalg}. For both algorithms, an important computational primitive we make use of at several points is the following semidefinite program given a set of points $\mc{Z} = \{Z_i\}_{i = 1}^n$ and $\delta > 0$:
\begin{equation}
    \label{eq:gestsdp} \tag{MT}
    \min_{s \in \Delta_{\delta}} \lambda_{\text{max}} \lprp{\sum_{i = 1}^n s_i Z_i Z_i^\top}.
\end{equation}
In what follows, we use \ref{eq:gestsdp}$(\mc{Z}, \delta)$ to denote the program \ref{eq:gestsdp} instantiated with inputs $\mc{Z}$ and $\delta$. Semidefinite programs of the form featured in the above display were used to build near linear time algorithms for robust mean estimation~\citep{cheng2019high} by reducing the problem to a packing/covering SDP and subsequently exploiting fast positive semidefinite programming solvers of \citet{peng2012faster}. The dual solution to the above program is then used to refine an estimate of the mean. In contrast, our approach only utilizes primal solutions to the above SDP allowing us to simplify the reduction to packing/covering SDP (See \cref{sec:fstsolv}) resulting in faster solvers for a set of vectors $\mc{Z}$ of any scale as opposed to \citep{cheng2019high} where $\norm{Z_i} \leq O(\sqrt{d})$. 

\subsection{Algorithm for Heavy Tailed Robust Linear Regression}
\label{ssec:htalg}

In the heavy tailed setting, our approach builds on the iterative robust gradient estimation framework from \citep{prasad2018robust,diakonikolas2019sever} where in each step, we construct a robust estimate of the gradient of the function $f(w) = \mb{E} [(Y - \inp{w}{X})^2]$. In each iteration, $t$, we compute the gradient at each of the sample points and then use a solution to \ref{eq:gestsdp} to obtain a reliable estimate of the gradient of $f$. We then use this gradient to improve our estimate until convergence. 

Note that in the non-adversarial scenario, the average of the sample gradients, $G_i(w_t) = (\inp{X_i}{w_t} - Y_i)X_i$ provides a good estimate of the population gradient, defined to be $G^*(w) = \Sigma (w - w^*)$. However, in the robust setting this is no longer true. As a remedy to such issues, \ref{eq:gestsdp}, employs a set of weights which allows one to exclude samples which adversely affect the value of the estimated gradient. The main observation is that to effect a mild change in the mean of the sample gradients, an outlier must have an outsized effect on the sample second moment leading to lower weight being assigned to the outlier. Our algorithm is formally described in \cref{alg:htlinreg} and our gradient estimation procedure is described in \cref{alg:htgest} where we simply use the primal weights from a solution to \ref{eq:gestsdp} to obtain our gradient estimate.
\begin{algorithm}[H]
\caption{Heavy Tailed Linear Regression}
\label{alg:htlinreg}
\begin{algorithmic}[1]
\State \textbf{Input: } Set of sample points $\bm{Z} =  \{(X_i, Y_i)\}_{i = 1}^n$, Outlier Fraction $\eta$, Length of Parameter Vector $\norm{w^*}$
\State $T \gets \nit$
\State $w_0 \gets 0$
\For {$t = 0:T$}
    \State $g_t \gets \gest (\bm{Z}, \eta, w_t)$
    \State $w_{t + 1} \gets w_t - g_t$
\EndFor
\State \textbf{Return: } $w_{T + 1}$
\end{algorithmic}
\end{algorithm}
\begin{algorithm}[H]
\caption{Gradient Estimation}
\label{alg:htgest}
\begin{algorithmic}[1]
\State \textbf{Input: } Set of sample points $\bm{Z} =  \{(X_i, Y_i)\}_{i = 1}^n$, Outlier Fraction $\eta$, Current Estimate $w$
\State $\bm{G} = \{G_i (w) = (\inp{X_i}{w} - Y_i)X_i\}$
\State $s \gets (1 + 5\eta)$-approximate solution to \ref{eq:gestsdp}$(\bm{G}, 10\eta)$ 
\State $g \gets \mb{E}_s [G_i(w)]$
\State \textbf{Return: } $g$
\end{algorithmic}
\end{algorithm}

\subsection{Algorithm for Sub-Gaussian Robust Linear Regression}
\label{ssec:sgalg}

In this subsection, we present our algorithm for robust linear regression in the setting where the covariates are assumed to be sub-Gaussian with identity covariance (see \cref{as:sgassump}).
The estimate $w^\dagger$ approximates $w^*$ up to an error of $O(\sqrt{\eta} \sigma)$. Interestingly, we further establish that refining the estimate to $O(\eta \log 1 / \eta)$ simply requires solving \ref{eq:gestsdp} one more time and using the average of sample gradients as a correction factor thus obtaining our refined rates. The procedure is detailed in \cref{alg:sgalg}. Notably, this procedure can also be used to simplify algorithms for robust mean estimation under sub-Gaussian assumptions.

\begin{algorithm}[H]
\caption{Sub-gaussian Robust Linear Regression}
\label{alg:sgalg}
\begin{algorithmic}[1]
\State \textbf{Input: } Set of sample points $\bm{Z} =  \{(X_i, Y_i)\}_{i = 1}^n$, Outlier Fraction $\eta$, Length of Parameter Vector $\norm{w^*}$
\State $n_1 \gets \tO(d / \eta)$
\State $\bm{Z}_1 \gets \text{Random $n_1$ samples from $\bm{Z}$}$
\State $ \bm{\tilde Z}_1 \gets \{(X_i, Y_i) \in \bm{Z}_1: \norm{X_i} \leq O(\sqrt{d})\}$
\State $w^\dagger \gets \text{Heavy Tailed Linear Regression}(\bm{\tilde Z}_1, \eta, \norm{w^*})$
\State $\bm{Z}_2 \gets \{(X_i, (Y_i - \inp{X_i}{w^\dagger})) : (X_i, Y_i) \in \bm{Z} \setminus \bm{Z}_1\}$
\State $\bm{G} = \{G_i = Y_iX_i : (X_i, Y_i) \in \bm{Z}_2\}$
\State $s \gets (1 + 3\eta)$-approximate solution to \ref{eq:gestsdp}$(\bm{G}, 6\eta)$
\State $\hat{w} \gets w^\dagger + \mb{E}_s [G_i]$
\State \textbf{Return: } $\hat{w}$
\end{algorithmic}
\end{algorithm}

\section{Analysis under Deterministic Assumptions}

In this section we provide the analysis of our method under deterministic assumptions on the distribution of the good samples. 

\subsection{Analysis for Heavy Tailed Robust Linear regression}

We prove here the success of our approach in the heavy tailed case. The main steps involved in the proof are the following:
\begin{myenumerate}
    \item \textbf{Gradient Estimation:} We show that the vector $G$ output by \cref{alg:htgest} is a good estimate of the true gradient $G^*(w)$. See \cref{lem:htgapxtemp}.
    \item \textbf{Gradient Descent:} We prove that the gradient descent algorithm eventually converges to a good approximation of the regression vector $w^*$. See \cref{lem:gd}.
\end{myenumerate}
To prove the correctness of these two steps, we require the following condition to hold:
\begin{assumption}
\label{as:htd} 
For every $w \in \R^d$, there exists $S \subset [n]$ such that $\abs{S} \geq (1 - 5\eta)n$ satisfying:
    \begin{equation*}
        \norm{\mb{E}_S[G_i(w)] - G^*(w)} \leq O(\sqrt{\eta} (\norm{w - w^*} + \sigma)) \text{ and } \norm{\mb{E}_S[ G_i(w) G_i(w)^\top]} \leq O(\norm{w - w^*}^2 + \sigma^2)
    \end{equation*}
\end{assumption}
Deriving the concentration tools to prove theses conditions is the object of \cref{sec:htconc} where it is shown they uniformly hold, over all $w \in \mb{R}^d$, with probability at least $0.95$ (see \cref{lem:htdet}).


\subsubsection{Gradient estimation}
First we prove the correctness of the \textbf{Gradient Estimation} step of \cref{alg:htgest}. 
\begin{lemma}
    \label{lem:htgapxtemp}
Under \cref{as:htd}, there is a universal constant $\cge>0$ such that for any $w\in\bbR^d$, any $(1 + 5\eta)$-approximate solution, $s$, to \ref{eq:gestsdp}$(\bm{G}, 10\eta)$ for $\bm{G} = \{G_i(w) = (\inp{X_i}{w} - Y_i)X_i\}_{i = 1}^n$, satisfies:
    \begin{equation*}
        \norm{\mb{E}_s [G_i(w)] - G^*(w)} \leq \cge \sqrt{\eta}  \lprp{\norm{w - w^*} + \sigma}.
    \end{equation*}
\end{lemma}

\begin{proof}
This result follows from the simple idea that  two distributions with small TV distance and bounded second moments have close means.
Consider any $\tilde s\in\Delta_{5\eta}$.  Applying the Cauchy-Schwarz inequality for any direction $\Vert v \Vert = 1$ and a coupling argument between $s$ and $\tilde s$, we get:
\[
\langle \bbE_s[G_i]-\bbE_{\tilde s} [G_i],v\rangle \leq \sqrt{\bbE [\langle G_s-G_{\tilde s},v\rangle ^2] }\sqrt{\text{Dist}_{TV}(s,\tilde s)},
\]
where $G_s$ denotes the random vector $G_s=G_i$ with probability $s_i$ and $G_s, G_{\tilde s}$ are coupled. We first note that 
\[
\bbE [\langle G_s-G_{\tilde s},v\rangle ^2]\leq 2\bbE [\langle G_s,v\rangle ^2] +2\bbE [\langle G_{\tilde s},v\rangle ^2] 
 \leq 2(1+5\eta)\OPT^*_{5\eta}+2\bbE [\langle G_{\tilde s},v\rangle ^2]\leq 14 \bbE[ \langle G_{\tilde s},v\rangle ^2], 
\]
since $s$ approximately solves the SDP and $\tilde s$ is feasible for this SDP. A direct application of \cref{lem:weighttv} yields $\text{Dist}_{TV} (s, \tilde s) \leq 15 \eta$. Therefore, multiplying the two bounds and maximizing over all directions $\Vert v \Vert=1$ yields 
\[\norm{\bbE_s [G_i]-\bbE_{\tilde s} [G_i]} =O\lprp{ \sqrt{\eta\lmax\lprp{\bbE_{\tilde s}[G_iG_i^\top]}}}.
 \]
Finally using the triangle inequality, we obtain 
 \[
 \norm{ \bbE_s [G_i]-G^*} =O\lprp{\sqrt{\eta \lmax\lprp{\bbE_{\tilde s}[G_iG_i^\top]}}} +  \norm{ \bbE_{\tilde s} [G_i] - G^*}.
 \]
We conclude by applying this result to the weight vector $\tilde s$ defined as $\tilde s =\frac{\bm{1}_{S}}{\abs{S}}$ where $S$ is the set given in \cref{as:htd}.
\end{proof}

\subsubsection{Gradient Descent}
Here we show that the gradient descent (i.e. \cref{alg:htlinreg}) can produce an accurate estimate of the underlying parameter vector, $w^*$, when used with the gradient estimates of \cref{alg:htgest}.
The proof amounts to analyzing gradient descent dynamics with deterministic noise in the gradients. 
\begin{lemma}
\label{lem:gd}
Suppose that \cref{as:htassump,as:htd} hold.
Assuming that $1/\cn> 2\cge\sqrt{\eta}$, \cref{alg:htlinreg}  outputs an estimate $w_{T}$ satisfying,
    \begin{equation*}
        \norm{w_{T}-w^*} \leq O(\cn \sigma \sqrt{\eta}),
    \end{equation*}
    in $T= O\big(\cn \log\big(\frac{\norm{w^*-w_0}}{\cge  \sqrt{\eta} \sigma \cn}\big)\big)$ iterations, where $\cge$ is the same constant as in \cref{lem:htgapxtemp}.
\end{lemma}

\begin{proof}
    Noting $G^*(w) = \Sigma(w - w^*)$, we consider the evolution of the sequence $w_{t+1}-w^*$:
    \begin{equation*}
        w_{t}-w^* = w_{t-1}-w^*- \bbE_{s}[G_{i}(w_{t-1})] = (I- \Sigma)(w_{t-1}-w^*) -  e_{t-1},
    \end{equation*}
    where the error sequence $e_t = \mb{E}_s [G_i(w_{t})] - G^*(w_{t})$ uniformly satisfies $\norm{e_t} \leq \sqrt{\eta} \cge \cdot \lprp{\norm{w_t - w^*} + \sigma}$ from \cref{lem:htgapxtemp}. Taking norms of the equation and applying the triangle inequality shows that, 
\begin{equation*}
    \norm{w_t-w^*} \leq (\norm{I- \Sigma} +  \sqrt{\eta} \cge) \norm{w_{t-1}-w^*} +  \sqrt{\eta} \cge \sigma.
\end{equation*}
Unrolling the recursive inequality shows that, 
    \begin{align*}
       \norm{w_{t}-w^*} \leq (\norm{I- \Sigma} +  \sqrt{\eta} \cge)^t \norm{(w_0-w^*)}  +  \sqrt{\eta} \cge \sigma \sum_{i=0}^{t-1} (\norm{I- \Sigma} +  \sqrt{\eta} \cge)^i. 
    \end{align*}
Recall that we normalize $\norm{\Sigma}=1$ and hence, $\norm{I- \Sigma} +  \sqrt{\eta} \cge \leq 1-(\frac{1}{\kappa} - \cge\sqrt{\eta}) \leq 1-\frac{1}{2\kappa}\leq e^{-1/2\cn}< 1$ since we assume that $1/\cn> 2\cge\sqrt{\eta}$.
Using this choice of step-size and summing the geometric series (to $\infty$) implies that,
\begin{equation*}
    \norm{w_t-w^*} \leq e^{-t/(2 \kappa)} \norm{w_0-w^*} +  4\cge \sqrt{\eta} \sigma \kappa.
\end{equation*}
Hence choosing $t = O\big(\cn \log\big(\frac{\norm{w^*-w_0}}{\cge \sqrt{\eta} \sigma \cn}\big) \big)$ ensures that the upper bound is $O(\cn \sigma  \sqrt{\eta} )$. 
\end{proof}

\subsection{Analysis for Sub-Gaussian Robust Linear Regression}

From \cref{lem:gd}, we can assume that we have an initial estimate $w^\dagger$ such that $\norm{w^\dagger - w^*} \leq O(\sqrt{\eta})$. Therefore, by subtracting out $\inp{X_i}{w^\dagger}$ from all the data points, our problem reduces to a setting where we additionally have $\norm{w^*} \leq O(\sqrt{\eta})$. We make this formal in the following assumption:
\begin{assumption}
    \label{as:sgiest}
    We assume that there exists a constant $\nu$, such that $\norm{w^*} \leq \nu \sigma \sqrt{\eta}$.
\end{assumption}
To avoid dealing with the randomness of clean samples, we require in addition the following condition:
\begin{assumption}
\label{as:sgd} 
There exists a set $S \subset [n]$ with $\abs{S} \geq (1 - 3\eta)n$ such that for any $T \subset S$ with $\abs{T} \geq (1 - 10\eta)n$:
    \begin{equation*}
        \norm{\mb{E}_T [Y_i X_i] - w^*} \leq O(1) \cdot \sigma \eta \log{1 / \eta} \text{ and } \norm{\mb{E}_T [Y_i^2 X_i X_i^\top] - \sigma^2 \cdot I} \leq O(1) \cdot \sigma^2 \eta \log^2 1 / \eta.
    \end{equation*}
\end{assumption}
These conditions are shown to hold in~\cref{ap:sgconc}. We present now the following lemma which states that solving $\eqref{eq:gestsdp}$ one more time is enough to refine the estimation rate to $O(\eta \log(1/\eta))$.
\begin{lemma}
    \label{lem:sgfintemp2}
    Suppose \cref{as:sgd,as:sgiest} hold for $\bm{Z} = \{(X_i, Y_i)\}_{i = 1}^n$. Let ${\bm{G} = \{G_i = Y_i X_i : (X_i, Y_i) \in \bm{Z}\}}$. Then, a $(1 + 3\eta)$-approximate solution to \ref{eq:gestsdp}$(\bm{G}, 6\eta)$, $s$, satisfies:
  \[
  \norm{\mb{E}_s [G_i] - w^*} \leq O(1) \cdot \sigma \eta \log 1 / \eta.\]
\end{lemma}
\begin{proof}
 First, using \cref{as:sgd} we easily see that
    \[
        \lambda_{\max}\lprp{\mb{E}_s [G_i G_i^\top]} \leq \sigma^2 + O(1) \cdot \sigma^2 \eta \log^2 1 / \eta.
    \]
 Then we decompose $s$ as $s = \sum_{\alpha \in \mc{E}_{6\eta}} \gamma_\alpha \alpha$ by \cref{lem:extpts} with $\sum_{\alpha \in \mc{E}_{6\eta}} \gamma_\alpha = 1$. In addition, we define $\alpha_T = \sum_{i \in T} \alpha_i$. Using the triangle inequality we have:
    \begin{equation*}
        \bignorm{\mb{E}_s[G_i] - w^*} \leq \bignorm{\sum_{i \in S} s_i (G_i - w^*)} + \bignorm{\sum_{i \notin S} s_i (G_i - w^*)}.
    \end{equation*}
    We bound the first term as follows:
    \begin{align*}
        \bignorm{\sum_{i \in S} s_i (G_i - w^*)}\! \leq\! \!\sum_{\alpha \in \mc{E}_{6\eta}}\!\!\gamma_\alpha \bignorm{\sum_{i \in S}\!\! \alpha_i (G_i - w^*)} = \!\!\sum_{\alpha \in \mc{E}_{6\eta}} \!\!\gamma_\alpha \cdot \frac{1}{\alpha_S} \norm{\mb{E}_{\text{Supp}(\alpha) \cap S} (G_i - w^*)} \leq O(1)\cdot \sigma \eta \log 1 / \eta
    \end{align*}
    where the last inequality follows because we have $\text{Supp}(\alpha) \cap [n] \geq (1 - 10\eta)n$ and from \cref{as:sgd}. For the second term, we have for all $\norm{u} = 1$:
    \begin{align*}
        \mb{E}_s [\inp{u}{(G_i - w^*)} \bm{1} \lbrb{i \notin S}] &\leq (\mb{E}_s [\inp{u}{G_i - w^*}^2 \bm{1} \lbrb{i \notin S}])^{1/2} (\mb{E}_s [\bm{1} \lbrb{i \notin S}])^{1/2} \\
        &\leq 2 \sqrt{\eta} (\mb{E}_s [\inp{u}{G_i - w^*}^2 \bm{1} \lbrb{i \notin S}])^{1/2}\\
        &\leq 4\sqrt{\eta} (\mb{E}_s [\inp{u}{G_i }^2 \bm{1} \lbrb{i \notin S}])^{1/2} + 4\sqrt{\eta} (\mb{E}_s [\inp{u}{w^*}^2 \bm{1} \lbrb{i \notin S}])^{1/2}.
    \end{align*}
    We finally bound the last terms in the above display as follows:
           \begin{align*}
        \mb{E}_s [\inp{u}{G_i }^2 \bm{1} \lbrb{i \notin S}] &= \mb{E}_s [\inp{u}{G_i }^2] - \mb{E}_s [\inp{u}{G_i }^2 \bm{1} \lbrb{i \in S}] \\
        &\leq (1 + O(1) \eta \log^2 1 / \eta) \sigma^2 - \sum_{\alpha \in \mc{E}_{6\eta}} \gamma_\alpha \cdot \alpha_S \cdot \mb{E}_{\text{Supp} \cap S} [\inp{u}{G_i - w^*}^2 \bm{1} \lbrb{i \in S}] \\
        &\leq \sigma^2 \lprp{1 + O(1) \eta \log^2 1 / \eta -(1 - O(1) \cdot \eta \log^2 1 / \eta)} \leq O(1)\cdot \sigma^2 \eta \log^2 1 / \eta,
    \end{align*}
    where the last inequality follows from \cref{lem:extpts,as:sgd}. And 
    \[
        \mb{E}_s [\inp{u}{w^* }^2 \bm{1} \lbrb{i \notin S}] \leq\norm{w_*}^2   \mb{E}_s [\bm{1} \lbrb{i \notin S}]\leq 4\eta \norm{w_*}^2\leq 4\eta^{3/2},
    \]
    where the last inequality follows from  \cref{as:sgiest}.   Substituting the above bounds in the previous equations, we conclude the proof of the lemma.
\end{proof}
\vspace{-.5cm}
\section{Concentration}
\label{sec:htconc}

In this section, we sketch the main arguments establishing the deterministic conditions in \cref{as:htd} required for the success of \cref{alg:htlinreg}.  Since, the corresponding proofs for the sub-Gaussian setting are conceptually and technically simpler, we defer them to \cref{ap:sgconc}. Informally, under our assumptions, we are required to show for any estimate, $w$, the existence of a good set $S \subseteq [n]$ exhibiting strong concentration properties for $\mb{E}_S [G_i(w)]$ and $\mb{E}_S [G_i(w) G_i(w)^\top]$. Before we begin, we first note that the following decomposition is useful for the subsequent analysis:
\begin{align}
    \label{eq:expandmean}
    & \mb{E}_S [G_i(w)] = (\mb{E}_S [X_i X_i^\top]) \underbrace{(w - w^*)}_{u} - \mb{E}_S [\epsilon_i X_i]  \quad \text{ and} \\ 
    \label{eq:expandcov}
    & \mb{E}_S [G_i(w) G_i(w)^\top] = \mb{E}_S [\inp{X_i}{ \underbrace{(w-w^*)}_{u}}^2 X_i X_i^\top - 2 \epsilon_i \inp{X_i}{\underbrace{(w-w^*)}_{u}} X_i X_i^\top  + \epsilon_i^2 X_i X_i^\top]
\end{align} 

\vspace{-.2cm}
We now establish control on the first term in \cref{eq:expandmean} through the following lemma which shows that the empirical second moment $\mb{E}_S [X_i X_i^\top]$ is close to the true second moment matrix for all suitably large sets, $S$\footnote {The set $S$ is eventually chosen to exclude points corrupted by the adversary.}.
\begin{lemma}
    \label{lem:htmeanconc}
    Let $\{X_i\}_{i = 1}^n$ satisfy \cref{as:htassump}. Then, there exists a universal constant $c$ such that  if $n \geq c \log(\frac{4d}{\sqrt{\eta}}) \frac{d}{\eta}$, with probability at least $1 - 1 / d^2$ for any set of $(1 - 10\eta) n$ samples, $S$, we have:
    \begin{equation*}
        \Sigma - O(\sqrt{\eta}) \cdot I \preccurlyeq \mb{E}_S[ X_i X_i^\top] \preccurlyeq \Sigma + O(\sqrt{\eta}) \cdot I. 
    \end{equation*}
\end{lemma}
\begin{itemize}
\item Showing the upper bound is straightforward -- we simply note that $\sum_{i \in S} X_i X_i^\top \preceq \sum_{i=1}^n X_i X_i$ and that $\frac{\abs{S}}{n} = 1+O(\eta)$ -- hence we can simply apply the matrix Bernstein inequality to the latter sum to show the result uniformly over all sets $S$. 

\item The proof of the lower bound proceeds by first noting that it suffices to show that for any $S$, $ \mb{E}_S [Z_i] \geq u^\top\Sigma u - O(\sqrt{\eta})$ where $Z_i = \inp{X_i}{u}^2$ for any $u$ in an $\delta$-net of resolution $O(\frac{\sqrt{\eta}}{d})$, since at this scale the error of replacing $u \in S^{d-1}$ by its closest approximant in the net is negligible. Second, using the Bernstein inequality, for a fixed direction $u$, we argue that with high probability at least $10 \eta$ points lie in the interval $(q, \infty)$, where $q$ is defined as $\mb{P} (Z_i \in (q, \infty)) = 20\eta$. We then establish a bound on the following random variable $Z = \mb{E}_{[n]} [Z_i \bm{1} \lbrb{Z_i \leq q}]$. From our hypercontractivity assumptions, we get $\mb{E} [Z_i \bm{1} \lbrb{Z_i \leq q}] \geq u^\top \Sigma u - O(\sqrt{\eta})$. We exploit our hypercontractivity assumptions again to establish that $q \leq O(1 /\sqrt{\eta})$. As a consequence, we may apply the Bernstein inequality to the random variable $Z$ to obtain $Z \geq u^\top \Sigma u - O(\sqrt{\eta})$ w.h.p. and since the elements, $Z_i \in (q, \infty)$ contribute maximally to $\mb{E}_{[n]}[Z_i]$, the previous expression establishes a lower bound on $\mb{E}_S [Z_i]$ as at least $10 \eta n$ of the $Z_i$ are in $(q, \infty)$. Union-bounding over these high-probability events and the $\delta$-net concludes the argument. 
\end{itemize}

To finish controlling \ref{eq:expandmean} only bounding the term $ \mb{E}_S[\epsilon_i X_i]$ remains. In the heavy-tailed setting, it suffices to use the boundedness of the covariance and exploit the randomness in $\epsilon_i$ to exhibit the result. This allows us to obtain a bound of $O(\sqrt{\eta} \sigma)$ for $\mb{E}_S [\epsilon_i X_i]$ for all suitable large $S$. Since these results follow from straightforward generalizations of concentration techniques for robust mean estimation, we defer these statements and their proofs to \cref{ap:htc} (See \cref{lem:hterrconc}).

We now state the main lemma to establish \cref{eq:expandcov}, which controls the leading term. Control of the $\epsilon_i^2 X_i X_i^\top$ term follow from similar techniques as \cref{lem:htmeanconc} and the $2 \epsilon_i \langle X_i, u \rangle X_i X_i^\top$ term follows from the Cauchy-Schwarz inequality.

\begin{lemma}
    \label{lem:htcovconcset}
    Let $\mc{D}$ satisfy \cref{as:htassump}. Then, for $n = \tO (d / \eta)$ samples from $\mc{D}$, we have with probability at least $1 - 1/d$ that for every $\norm{u} = 1$, there exists $S \subset [n]$ with $\abs{S} \geq (1 - \eta / 6)n$ such that:
    \begin{equation*}
       \mb{E}_S [\inp{X_i}{u}^2 X_i X_i^\top] \preccurlyeq \cff I,
    \end{equation*}
    and furthermore, for all $i \in S$, we have $\abs{\inp{X_i}{u}} \leq 40 / \eta$.
\end{lemma}

Establishing this result requires several steps,
\begin{itemize}
    \item First by exploiting convex duality we can see that:
    \begin{equation}
       \!\!  \min_{\alpha \in \Delta_{\!\frac{\eta}{25}}} \!  \max_{\substack{M \succcurlyeq 0,\\ \tr{M} = 1}}    \inp*{\sum_{i \in T} \alpha_i \inp{X_i}{u}^2 X_iX_i^\top}{M} \!=\!\max_{\substack{M \succcurlyeq 0,\\ \tr{M} = 1}} \!  \min_{\alpha \in \Delta_{\!\frac{\eta}{25}}} \!\inp*{\sum_{i \in T} \!\alpha_i \inp{X_i}{u}^2 X_iX_i^\top\!\!}{M\!}, \label{eq:pfsketchminmax}
    \end{equation} 
    Apriori the weights $\alpha_i$ may depend on the vector $u$ and can be chosen to discard terms from the sum for which $\langle X_i, u \rangle^2$ is large. Exchanging the $\min$ and $\max$ allows $\alpha$ to also depend on the matrix $M$ and admits the possibility they may be chosen to discard terms for which the matrix $M$ has strong overlap with $X_i$. The price of exchanging the $\min$ and $\max$, is that in the left-side of \cref{eq:pfsketchminmax} the optimal $M$ is always achieved at a rank-one matrix; however in order to appeal to convex duality we must lift the program to an SDP which in general results in a semi-definite matrix $M$ appearing on the right-hand side.
    \item As a stepping stone we consider the simpler case where $M = vv^\top$ is rank-$1$ and exhibit a set of weights $\alpha_i$ (which may depend $M$) which can avoid $X_i$ which are strongly correlated with $u$ and $v$. This is formally stated in:
\begin{lemma}
    \label{lem:htcovconcr1}
    Let $\mc{D}$ satisfy \cref{as:htassump}. Then, for $n = \tO (d / \eta)$ samples from $\mc{D}$, we have with probability at least $1 - 1/d$, that for all $\norm{u} \leq 1$ and $\norm{v} \leq 1$:
    \begin{gather*}
        \mb{E}_{[n]} [\sum_{i = 1}^n \inp{X_i}{u}^2 \inp{X_i}{v}^2 \bm{1} \lbrb{\abs{\inp{X_i}{u}} \leq 800 / \sqrt{\eta} \wedge \abs{\inp{X_i}{v}} \leq 800 / \sqrt{\eta}}] \leq \ccov \quad \text{ and } \\
        \mb{E}_{[n]} [\bm{1} \lbrb{\abs{\inp{X_i}{u}} \leq 40/\sqrt{\eta} \wedge \abs{\inp{X_i}{v}} \leq 40 / \sqrt{\eta}}] \geq (1 - \eta / 100),
    \end{gather*}
    for some absolute constant $\ccov$.
\end{lemma}
We first argue that given pair $u$ and $v$, we may discard problematic terms in the sum for which $X_i$ overlaps strongly with $u$ or $v$ without changing the number of non-zero terms in the sum by more than a factor of $O(\eta n)$. This allows us to focus our attention on terms satisfying $\abs{\inp{X_i}{u}}, \abs{\inp{X_i}{v}} \leq O(1 / \sqrt{\eta})$. To obtain the above condition uniformly over all $u$ and $v$, we use a sufficiently fine grid over $(u, v)$ such that approximation error is negligible and apply a union bound over all the elements in the grid. 

\item The final and most technical step is to demonstrate that the previous argument which is restricted to rank-one matrices $M=vv^\top$ can be lifted to arbitrary normalized p.s.d. matrices. We formally state this in, 
\begin{lemma}
    \label{lem:htcovconc}
    Let $\mc{D}$ satisfy \cref{as:htassump}. Then, for $n = \tO(d / \eta)$ samples from $\mc{D}$, we have with probability at least $1 - 1/d$, that for every $\norm{u} = 1$ and p.s.d. matrix $M$ with $\Tr{M} = 1$:
    \begin{gather*}
        \mb{E}_{[n]} [\inp{X_i}{u}^2 \inp{X_iX_i^\top}{M} \bm{1} \lbrb{\abs{\inp{X_i}{u}} \leq 40 / \sqrt{\eta} \wedge \inp{X_iX_i^\top}{M} \leq 400^2 / \eta}] \leq \cfour \ \ \text{ and } \\
        \mb{E}_{[n]} [\bm{1} \lbrb{\abs{\inp{X_i}{u}} \leq 40 / \sqrt{\eta} \wedge \inp{X_iX_i^\top}{M} \leq 400^2 / \eta}] \geq (1 - \eta / 25),
    \end{gather*}
    for some absolute constant $\cfour$.
\end{lemma}
The argument to do this uses an intricate Gaussian rounding argument which generalizes arguments used previously in the context of heavy tailed mean estimation \citep{lecue2019robust} where a rounding argument was used to obtain large deviation bounds of the points $X_i$ in the ``direction'' $M$. However, these results do not suffice for our application as these only provide us a bound of $O(1 / \eta^2)$ as opposed to a constant. We generalize these arguments to obtain variance bounds on $X_i$ in the ``direction'' $M$. The full details of the proof are presented in \cref{ap:htc}.
\end{itemize}

Combining these results establishes \cref{lem:htcovconcset}. Finally, we state the main result of this section regarding the gradient estimates, $G_i(w) = (\inp{X_i}{w} - Y_i)X_i$, for any candidate weight vector $w$, which can be proved by combining the aforementioned results:

\begin{lemma}
\label{lem:htdet}
    Let $\mc{D}$ satisfy \cref{as:htassump}. Given $n = \tO(d / \eta)$ $(2\eta)$-corrupted samples from $\mc{D}$, with probability at least $0.95$, there exists for every $w \in \R^d$, a subset $S \subset [n]$ such that $\abs{S} \geq (1 - 5\eta)n$ satisfying:
    \begin{equation*}
        \norm{\mb{E}_S[G_i(w)] - G^*(w)} \leq O(\sqrt{\eta} (\norm{w - w^*} + \sigma)), \ \  \norm{\mb{E}_S[ G_i(w) G_i(w)^\top]} \leq O(\norm{w - w^*}^2 + \sigma^2),
    \end{equation*}
    and furthermore, for all $i \in S$, we have $\norm{G_i(w)} \leq O \lprp{\sqrt{\frac{d}{\eta}} \norm{w - w^*} + \sigma \sqrt{\frac{d}{\eta}}}$.
\end{lemma}
With this deterministic result in hand we can directly conduct the algorithmic analysis. 

\phantomsection
\addcontentsline{toc}{section}{References}
\bibliographystyle{plainnat}
\bibliography{custom}

\begin{thebibliography}{39}
\providecommand{\natexlab}[1]{#1}
\providecommand{\url}[1]{\texttt{#1}}
\expandafter\ifx\csname urlstyle\endcsname\relax
  \providecommand{\doi}[1]{doi: #1}\else
  \providecommand{\doi}{doi: \begingroup \urlstyle{rm}\Url}\fi

\bibitem[Bakshi and Prasad(2020)]{bakshi2020robust}
Ainesh Bakshi and Adarsh Prasad.
\newblock Robust linear regression: Optimal rates in polynomial time, 2020.

\bibitem[Balakrishnan et~al.(2017)Balakrishnan, Du, Li, and
  Singh]{balakrishnan2017computationally}
S.~Balakrishnan, S.~S. Du, J.~Li, and A.~Singh.
\newblock Computationally efficient robust sparse estimation in high
  dimensions.
\newblock In \emph{Proceedings of the 2017 Conference on Learning Theory},
  volume~65 of \emph{Proceedings of Machine Learning Research}, pages 169--212.
  PMLR, 07--10 Jul 2017.

\bibitem[Bhaskaran et~al.(2013)Bhaskaran, Gasparrini, Hajat, Smeeth, and
  Armstrong]{bhaskaran2013time}
K.~Bhaskaran, A.~Gasparrini, S.~Hajat, L.~Smeeth, and B.~Armstrong.
\newblock Time series regression studies in environmental epidemiology.
\newblock \emph{International journal of epidemiology}, 42\penalty0
  (4):\penalty0 1187--1195, 2013.

\bibitem[Bhatia et~al.(2015)Bhatia, Jain, and Kar]{bhatia2015robust}
K.~Bhatia, P.~Jain, and P.~Kar.
\newblock Robust regression via hard thresholding.
\newblock In \emph{Advances in Neural Information Processing Systems 28}, pages
  721--729, 2015.

\bibitem[Bhatia et~al.(2017)Bhatia, Jain, Kamalaruban, and
  Kar]{bhatia2017consistent}
K.~Bhatia, P.~Jain, P.~Kamalaruban, and P.~Kar.
\newblock Consistent robust regression.
\newblock In \emph{Advances in Neural Information Processing Systems 30}, pages
  2110--2119, 2017.

\bibitem[Charikar et~al.(2017)Charikar, Steinhardt, and
  Valiant]{charikar2017learning}
M.~Charikar, J.~Steinhardt, and G.~Valiant.
\newblock Learning from untrusted data.
\newblock In \emph{Proceedings of the 49th Annual ACM SIGACT Symposium on
  Theory of Computing}, pages 47--60, 2017.

\bibitem[Cheng et~al.(2019)Cheng, Diakonikolas, and Ge]{cheng2019high}
Y.~Cheng, I.~Diakonikolas, and R.~Ge.
\newblock High-dimensional robust mean estimation in nearly-linear time.
\newblock In \emph{Proceedings of the Thirtieth Annual ACM-SIAM Symposium on
  Discrete Algorithms}, pages 2755--2771. SIAM, 2019.

\bibitem[Diakonikolas and Kane(2019)]{diakonikolas2019recent}
I.~Diakonikolas and D.~M Kane.
\newblock Recent advances in algorithmic high-dimensional robust statistics.
\newblock \emph{arXiv preprint arXiv:1911.05911}, 2019.

\bibitem[Diakonikolas et~al.(2016)Diakonikolas, Kamath, Kane, Li, Moitra, and
  Stewart]{diakonikolas2016robust}
I.~Diakonikolas, G.~Kamath, D.~M. Kane, J.~Li, A.~Moitra, and A.~Stewart.
\newblock Robust estimators in high dimensions without the computational
  intractability.
\newblock In \emph{Foundations of Computer Science (FOCS)}, 2016.

\bibitem[Diakonikolas et~al.(2017)Diakonikolas, Kamath, Kane, Li, Moitra, and
  Stewart]{diakonikolas2017being}
I.~Diakonikolas, G.~Kamath, D.~M. Kane, J.~Li, A.~Moitra, and A.~Stewart.
\newblock Being robust (in high dimensions) can be practical.
\newblock In \emph{Proceedings of the 34th International Conference on Machine
  Learning}, volume~70 of \emph{Proceedings of Machine Learning Research},
  pages 999--1008. PMLR, 2017.

\bibitem[Diakonikolas et~al.(2018{\natexlab{a}})Diakonikolas, Kamath, Kane, Li,
  Moitra, and Stewart]{diakonikolas2018robustly}
I.~Diakonikolas, G.~Kamath, D.~M Kane, J.~Li, A.~Moitra, and A.~Stewart.
\newblock Robustly learning a {G}aussian: Getting optimal error, efficiently.
\newblock In \emph{Proceedings of the Twenty-Ninth Annual ACM-SIAM Symposium on
  Discrete Algorithms}, pages 2683--2702. SIAM, 2018{\natexlab{a}}.

\bibitem[Diakonikolas et~al.(2018{\natexlab{b}})Diakonikolas, Kane, and
  Stewart]{diakonikolas2018list}
I.~Diakonikolas, D.~M Kane, and A.~Stewart.
\newblock List-decodable robust mean estimation and learning mixtures of
  spherical {G}aussians.
\newblock In \emph{Proceedings of the 50th Annual ACM SIGACT Symposium on
  Theory of Computing}, pages 1047--1060, 2018{\natexlab{b}}.

\bibitem[Diakonikolas et~al.(2019{\natexlab{a}})Diakonikolas, Kamath, Kane, Li,
  Steinhardt, and Stewart]{diakonikolas2019sever}
I.~Diakonikolas, G.~Kamath, D.~Kane, J.~Li, J.~Steinhardt, and A.~Stewart.
\newblock Sever: A robust meta-algorithm for stochastic optimization.
\newblock In \emph{Proceedings of the 36th International Conference on Machine
  Learning}, volume~97 of \emph{Proceedings of Machine Learning Research},
  pages 1596--1606. PMLR, 09--15 Jun 2019{\natexlab{a}}.

\bibitem[Diakonikolas et~al.(2019{\natexlab{b}})Diakonikolas, Kong, and
  Stewart]{diakonikolas2019efficient}
I.~Diakonikolas, W.~Kong, and A.~Stewart.
\newblock Efficient algorithms and lower bounds for robust linear regression.
\newblock In \emph{Proceedings of the Thirtieth Annual ACM-SIAM Symposium on
  Discrete Algorithms}, SODA ’19, page 2745–2754, 2019{\natexlab{b}}.

\bibitem[Dielman(2001)]{dielman2001applied}
T.~E Dielman.
\newblock \emph{Applied Regression Analysis for Business and Economics}.
\newblock Duxbury/Thomson Learning Pacific Grove, CA, 2001.

\bibitem[Dong et~al.(2019)Dong, Hopkins, and Li]{DBLP:conf/nips/DongH019}
Y.~Dong, S.~Hopkins, and J.~Li.
\newblock Quantum entropy scoring for fast robust mean estimation and improved
  outlier detection.
\newblock In \emph{Advances in Neural Information Processing Systems 3}, 2019.

\bibitem[Gao(2017)]{gao2017robust}
C.~Gao.
\newblock Robust regression via mutivariate regression depth, 2017.

\bibitem[Gauß(1809)]{gauss}
C.~F. Gauß.
\newblock \emph{Theoria Motus Corporum Coelestium in Sectionibus Conicis Solem
  Ambientium}.
\newblock Hamburgi : sumtibus Frid. Perthes et I. H. Besser, 1809, 1809.

\bibitem[Hopkins and Li(2018)]{hopkins2018mixture}
S.~Hopkins and J.~Li.
\newblock Mixture models, robustness, and sum of squares proofs.
\newblock In \emph{Proceedings of the 50th Annual ACM SIGACT Symposium on
  Theory of Computing}, pages 1021--1034. ACM, 2018.

\bibitem[Huber(1964)]{huber1964robust}
P.~J Huber.
\newblock Robust estimation of a location parameter.
\newblock \emph{The annals of mathematical statistics}, 35\penalty0
  (1):\penalty0 73--101, 1964.

\bibitem[Huber(2004)]{huber2004robust}
P.~J Huber.
\newblock \emph{Robust Statistics}, volume 523.
\newblock John Wiley \& Sons, 2004.

\bibitem[Isobe et~al.(1990)Isobe, Feigelson, Akritas, and
  Babu]{isobe1990linear}
T.~Isobe, E.~D Feigelson, M.~G Akritas, and G.~J. Babu.
\newblock Linear regression in astronomy.
\newblock \emph{The astrophysical journal}, 364:\penalty0 104--113, 1990.

\bibitem[Klivans et~al.(2018)Klivans, Kothari, and Meka]{klivans2018efficient}
A.~Klivans, P.~K. Kothari, and R.~Meka.
\newblock Efficient algorithms for outlier-robust regression.
\newblock In \emph{Proceedings of the 31st Conference On Learning Theory},
  volume~75 of \emph{Proceedings of Machine Learning Research}, pages
  1420--1430. PMLR, 2018.

\bibitem[Kothari et~al.(2018)Kothari, Steinhardt, and
  Steurer]{kothari2018robust}
P.~K Kothari, J.~Steinhardt, and D.~Steurer.
\newblock Robust moment estimation and improved clustering via sum of squares.
\newblock In \emph{Proceedings of the 50th Annual ACM SIGACT Symposium on
  Theory of Computing}, pages 1035--1046, 2018.

\bibitem[Lai et~al.(2016)Lai, Rao, and Vempala]{lai2016agnostic}
K.~A. Lai, A.~B. Rao, and S.~Vempala.
\newblock Agnostic estimation of mean and covariance.
\newblock In \emph{Foundations of Computer Science (FOCS)}, 2016.

\bibitem[Lecu{\'e} and Depersin(2019)]{lecue2019robust}
G.~Lecu{\'e} and J.~Depersin.
\newblock Robust subgaussian estimation of a mean vector in nearly linear time.
\newblock \emph{arXiv preprint arXiv:1906.03058}, 2019.

\bibitem[Lecu\'{e} and Mendelson(2016)]{lecue2016performance}
G.~Lecu\'{e} and S.~Mendelson.
\newblock Performance of empirical risk minimization in linear aggregation.
\newblock \emph{Bernoulli}, 22\penalty0 (3):\penalty0 1520--1534, 2016.

\bibitem[Ledoux and Talagrand(1991)]{ledoux1991probability}
M.~Ledoux and M.~Talagrand.
\newblock \emph{Probability in Banach Spaces}, volume~23.
\newblock Springer Science \& Business Media, 1991.

\bibitem[Legendre(1805)]{legendre}
A.~M. Legendre.
\newblock \emph{Nouvelles Methodes pour la Determination des Orbites des
  Cometes}.
\newblock F. Didot, 1805.

\bibitem[McDonald(2009)]{mcdonald2009handbook}
J.~H McDonald.
\newblock \emph{Handbook of Biological Statistics}, volume~2.
\newblock sparky house publishing Baltimore, MD, 2009.

\bibitem[Oliveira(2016)]{oliveira2016lower}
R.~I. Oliveira.
\newblock The lower tail of random quadratic forms with applications to
  ordinary least squares.
\newblock \emph{Probab. Theory Related Fields}, 166\penalty0 (3-4):\penalty0
  1175--1194, 2016.

\bibitem[Peng et~al.(2012)Peng, Tangwongsan, and Zhang]{peng2012faster}
R.~Peng, K.~Tangwongsan, and P.~Zhang.
\newblock Faster and simpler width-independent parallel algorithms for positive
  semidefinite programming.
\newblock \emph{arXiv preprint arXiv:1201.5135}, 2012.

\bibitem[Prasad et~al.(2018)Prasad, Suggala, Balakrishnan, and
  Ravikumar]{prasad2018robust}
A.~Prasad, A.~S. Suggala, S.~Balakrishnan, and P.~Ravikumar.
\newblock Robust estimation via robust gradient estimation, 2018.

\bibitem[Steinhardt et~al.(2018)Steinhardt, Charikar, and
  Valiant]{steinhardt2017resilience}
J.~Steinhardt, M.~Charikar, and G.~Valiant.
\newblock {Resilience: A Criterion for Learning in the Presence of Arbitrary
  Outliers}.
\newblock In \emph{9th Innovations in Theoretical Computer Science Conference
  (ITCS 2018)}, volume~94, 2018.

\bibitem[Suggala et~al.(2019)Suggala, Bhatia, Ravikumar, and
  Jain]{suggala2019adaptive}
A.~S. Suggala, K.~Bhatia, P.~Ravikumar, and P.~Jain.
\newblock Adaptive hard thresholding for near-optimal consistent robust
  regression.
\newblock In \emph{Proceedings of the Thirty-Second Conference on Learning
  Theory}, volume~99 of \emph{Proceedings of Machine Learning Research}, pages
  2892--2897. PMLR, 2019.

\bibitem[Tropp(2015)]{tropp2015introduction}
J.~A Tropp.
\newblock \emph{An Introduction to Matrix Concentration Inequalities},
  volume~8.
\newblock Now Publishers, Inc., 2015.

\bibitem[Tukey(1975)]{tukey1975mathematics}
J.~W. Tukey.
\newblock Mathematics and the picturing of data.
\newblock In \emph{Proceedings of the {I}nternational {C}ongress of
  {M}athematicians ({V}ancouver, {B}. {C}., 1974), {V}ol. 2}, pages 523--531,
  1975.

\bibitem[Vershynin(2012)]{vershynin2012introduction}
R.~Vershynin.
\newblock \emph{Introduction to the Non-Asymptotic Analysis of Random
  Matrices}.
\newblock Cambridge Univ. Press, Cambridge, 2012.

\bibitem[Zhu et~al.(2020)Zhu, Jiao, and Steinhardt]{zhu2020robust}
Banghua Zhu, Jiantao Jiao, and Jacob Steinhardt.
\newblock Robust estimation via generalized quasi-gradients, 2020.

\end{thebibliography}

\appendix

\section{Proof of Main Theorems}
\label{sec:mainthmpfs}

In this section, we complete the proofs of \cref{thm:htmain,thm:sgmain} given the results established in previous sections. 

\subsection{Proof of \texorpdfstring{\cref{thm:htmain}}{Theorem 3}}

First, we get that the conclusions of \cref{lem:htdet} holds for the sample points with probability at least $0.95$. Now, in light of \cref{lem:gd}, we only need to ensure that one obtains a good solution to \ref{eq:gestsdp} in for all iterations from $0$ to $T$. That is we need to ensure that we obtain a good solution to \ref{eq:gestsdp} in each iteration. \cref{lem:fastsdp} guarantees a good solution in a single run with probability $0.9$. To boost the probability, we will simply run the Algorithm in \cref{lem:fastsdp}, $O(\log T)$ times each iteration. We obtain a good solution in at least one run with probability at least $1 - 1 / (1000 T)$. To ensure we select a good solution, we compute $\lambda_{\text{max}} (\mb{E}_s [G_i G_i^\top])$ for each solution $s \in \Delta_{10\eta}$ and select the one with the smallest value. From the union bound, the probability that one has a good solution to \ref{eq:gestsdp} for all $T$ rounds is at least $0.999$. By taking a union bound over the events described by \cref{eq:prunsuccess}, \cref{lem:htdet} and the success of the solver designed for \ref{eq:gestsdp}, we conclude that \cref{alg:htlinreg} successfully returns an estimate, $\hat{w}$, satisfying $\norm{\hat{w} - w^*} \leq O(\sigma \kappa \sqrt{\eta})$ with probability at least $0.9$.

\qed

\subsection{Proof of \texorpdfstring{\cref{thm:sgmain}}{Theorem 5}}
We first analyze the set of samples, $\tilde{\bm{Z}}_1$. From \cref{lem:trunclem}, we get that $\mb{P}(\norm{X} \leq O(\sqrt{d})) \geq 0.99$. Therefore, we have via Hoeffding's inequality:

\begin{equation}
    \label{eq:prunsuccess}
    \mb{P} \left(\sum_{(X_i, Y_i) \in \bm{Z}_1} \bm{1} \lbrb{\norm{X_i} \leq O (\sqrt{d})} \geq 0.95 n_1\right) \geq 1 - \exp \lprp{- \frac{2n}{25^2}} \geq 1 - \frac{1}{100d^2}.
\end{equation}

We condition now on the above event. Furthermore, for $\norm{X_i} \leq O (\sqrt{d})$, we have that $(X_i, Y_i) \sim \mc{D}_S$ from \cref{lem:trunclem}. Since the adversary only corrupts, $\eta n_1$ sample points, at most $\eta n_1$ of the $0.95 n_1$ points in $\bm{\tilde{Z}}_1$ are corrupted. From \cref{lem:sgmoms,lem:trunclem}, we see that \cref{as:htassump} also hold for distributions satisfying \cref{as:sgassump} conditioned on the event $\{\norm{X_i} \leq O(\sqrt{d})\}$. Therefore, the result of \cref{thm:htmain} also holds for the set of samples $\bm{\tilde{Z}}_1$. Therefore, the estimate $w^\dagger$ in \cref{alg:sgalg} satisfies $\norm{w^\dagger - w^*} \leq O(\sqrt{\eta})$ with probability $9/10$. Therefore, \cref{as:sgassump,as:sgiest} hold for the set of samples, $\bm{Z}_2$. We now see that the conclusion of theorem holds true conditioned on \cref{as:sgd} holding for the set of samples $\bm{Z}_2$ which is guaranteed with probability at least $0.99$ and us obtaining a $(1 + 3\eta)$-approximate solution to \ref{eq:gestsdp} which happens with probability at least $9 / 10$ from \cref{lem:fastsdp}. Via a union bound, this event takes place with probability at least $2 / 3$ thus proving the theorem.

\qed

\section{Heavy-Tailed Concentration Results}
\label{ap:htc}

Here we collect the proofs of the relevant concentration arguments we use to show the deterministic conditions used in the gradient estimation step for heavy-tailed data. In what follows, we use $\mc{D}$ to denote the distribution over the pair $(X, Y)$. We begin by showing that truncating the data doesn't significantly affect its covariance or hypercontractivity constant.

\begin{lemma}
\label{lem:trunclem}
Consider a distribution, $\mc{D}$, satisfying all the conditions stated in \cref{as:htassump} except the boundedness assumption. Then, there exists a constant $\cp$ such that the conditional distribution of $\mc{D}$ on $S = \{\norm{X} \leq \cp \sqrt{d}\}$, $\mc{D}_S$ satisfies:

\begin{enumerate}
    \item $\mb{P} (\norm{X} \leq C_1 \sqrt{d}) \geq 0.99$
    \item $0.99 \Sigma \preccurlyeq \mb{E}_{\mc{D}_S} [XX^\top] \preccurlyeq 1.01 \Sigma$
    \item $\forall \norm{u} = 1,\ \mb{E}_{\mc{D}_S} [\inp{X}{u}^4] \leq 1.01 \hc \cdot (E_{\mc{D}_S} [\inp{X}{u}^2])^2$
    \item For $(X, Y) \thicksim \mc{D}_S$, we have that $\epsilon$ is independent of $X$ and has the same distribution as when $(X, Y) \thicksim \mc{D}$.
\end{enumerate}
\end{lemma}

\begin{proof}[Proof of \cref{lem:trunclem}]
    The fourth claim follows from the fact that $X$ and $\epsilon$ are independent random variables. Formally, we have for any measurable set $A$:
    \begin{equation*}
        \mb{P} \lprp{\epsilon \in A \: | \: \norm{X} \leq \cp \sqrt{d}} = \frac{\mb{P} \lprp{\epsilon \in A \cap \norm{X} \leq \cp \sqrt{d}}}{\mb{P} \lprp{\norm{X} \leq \cp\sqrt{d}}} = \frac{\mb{P} \lprp{\epsilon \in A} \cdot \mb{P}\lprp{\norm{X} \leq \cp \sqrt{d}}}{\mb{P} \lprp{\norm{X} \leq \cp \sqrt{d}}} = \mb{P} \lprp{\epsilon \in A}.
    \end{equation*}
    To prove the rest of the lemma, first note that we have $\mb{E} [\norm{X}^2] = \tr{\Sigma} \leq d$. From this, we obtain from Markov's inequality, for any $C\geq 0$, $\mb{P} \lprp{\norm{X} \leq C \sqrt{d}} \geq 1 - C^{-2}$. Let $S_C = \lbrb{\norm{X} \leq C\cdot \sqrt{d}}$. For the upper bound in the first claim, we have:
    \begin{equation*}
        \mb{E}_{\mc{D}_{S_C}} [XX^\top] = \frac{1}{\mb{P}_{\mc{D}} (S_C)} \cdot \mb{E}_{\mc{D}} [XX^\top \bm{1} \lbrb{X \in S_C}] \preccurlyeq \frac{1}{\mb{P}_{\mc{D}} (S_C)} \cdot \mb{E}_{\mc{D}} [XX^\top] \preccurlyeq \frac{1}{1 - C^{-2}} \cdot \Sigma.
    \end{equation*}
    For the lower bound, using the Cauchy-Schwarz inequality, we have for any vector $u$ of norm  $\norm{u} = 1$:
    \begin{equation*}
        \mb{E}_{\mc{D}} [\inp{X}{u}^2 \bm{1} \lbrb{X \notin S_C}] \leq (\mb{E}_{\mc{D}} [\inp{X}{u}^4])^{1/2} (\mb{E} [\bm{1} \lbrb{X \notin S_C}])^{1/2} \leq \frac{\sqrt{\hc}}{C} \cdot u^\top \Sigma u.
    \end{equation*}
    The lower bound now follows since:
    \begin{equation*}
        \mb{E}_{\mc{D}_{S_C}} [\inp{X}{u}^2 ] \geq \mb{E}_{\mc{D}} [\inp{X}{u}^2 \bm{1} \lbrb{X \in S_C}] = \mb{E}_{\mc{D}} [\inp{X}{u}^2] - \mb{E}_{\mc{D}} [\inp{X}{u}^2 \bm{1} \lbrb{X \notin S_C}] \geq \lprp{1 - \frac{\sqrt{\hc}}{C}} \cdot u^\top \Sigma u.
    \end{equation*}
    Finally, for the second claim, we have:
    \begin{equation*}
        \mb{E}_{\mc{D}_{S_C}} [\inp{X}{u}^4] \leq \frac{1}{1 - C^{-2}} \mb{E}_{\mc{D}} [\inp{X}{u}^4] \leq \frac{\hc}{1 - C^{-2}} (u^\top \Sigma u)^2 \leq \frac{\hc}{1 - C^{-2}} \lprp{1 - \frac{\sqrt{\hc}}{C}}^{-2} (\mb{E}_{\mc{D}_{S_C}} [\inp{X}{u}^2])^2.
    \end{equation*}
    By taking $C$ to be sufficiently large, we have the first three claims of the lemma.
\end{proof}

We now provide the proof of the first inequality in the deterministic condition.

\begin{proof}[Proof of \cref{lem:htmeanconc}]

The upper bound follows from the observation $\mb{E}_S [X_i X_i^\top] \preccurlyeq \mb{E}_{[n]} [X_i X_i^\top]$. In order to apply the matrix Bernstein inequality to control the sum $\mb{E}_{[n]} [W_i W_i^\top] - \E[W_i W_i^\top]$ in spectral norm, we first need the almost sure bound $R = \Norm{\frac{X_i X_i^\top}{n}} \leq \frac{\norm{X_i}^2}{n} = O(\frac{d}{n})$. Now, letting $Z = XX^\top$, where $X$ is a copy of an i.i.d. $X_i$, we have that the matrix variance satisfies $\frac{1}{n} \norm{\E[(Z-\E[Z])^2]} \leq \frac{1}{n} \norm{\E[Z^2]} \leq \frac{1}{n} \norm{\E[\norm{X}^2 XX^\top]} \leq O(\frac{d}{n})$ appealing to \cref{lem:linalgmean,lem:l4l2norm}. Finally, an application of the matrix Bernstein inequality \citep[see, e.g.,][]{tropp2015introduction} shows,
$\frac{1}{n} \sum_{i=1}^{n} X_i X_i^\top \preceq \E[XX^\top] + I \cdot O(\sqrt{\sigma^2 \log(2d/\delta)})+R \log(2d/\delta)) \preceq \Sigma + \sqrt{\eta} \cdot I$ by taking $\delta = \frac{1}{2d^2}$, and $n \geq c \frac{d \log (4d)}{\eta}$ for large-enough $c$. Finally, $\mb{E}_S [X_i X_i^\top] = \frac{n}{\abs{S}} \mb{E}_{[n]} [X_i X_i^\top \bm{1} \{i \in S\}] \preccurlyeq \frac{n}{\abs{S}} \mb{E}_{[n]} [X_i X_i^\top] $, and since,  $\frac{1}{1-10 \eta} \leq 1+O(\eta)$ we obtain that $\mb{E}_{S} [X_i X_i^\top] \preceq \Sigma + \sqrt{\eta} \cdot I$.

For the lower bound, consider a $\delta$-net over the unit sphere with resolution in the embedded $\ell_2$ distance, $G$. For a fixed $u_j \in G$, define the random variable $Z_i = \inp{X_i}{u_j}^2$. From L4-L2 hypercontractivity and the normalization of the spectral norm of $\Sigma$ we have that $\var(Z_i) \leq \E[Z_i^2]  =O(1)$. Now, consider the $1 - 20 \eta$ quantile of the random variable $Z_i$, $q$, defined as $\Pr[Z_i \leq q] = 1-20\eta$.  We can now once again appeal to L4-L2 hypercontractivity and Markov's inequality to see $20 \eta = \Pr[Z_i \geq q] \leq \frac{\E[Z_i^2]}{q^2} \leq O(\frac{1}{q^2})$, and $ q = O(\frac{1}{\sqrt{\eta}})$. Similarly using the Cauchy-Schwarz inequality, $\mb{E} [Z_i \bm{1} \lbrb{Z_i \leq q}] \geq \mb{E} [Z_i] - \E[Z_i \bm{1} \lbrb{Z_i > q}] \geq \E[Z_i] - \sqrt{\E[Z_i^2] \Pr[Z_i \geq q]} \geq \E[Z_i] - O(\sqrt{\eta})$. Since the random variable $\bm{1} \lbrb{Z_i > q}$ is almost surely bounded by $1$ and has variance bounded by $20 \eta$, applying the Bernstein inequality, we can see that that with probability at least $1-\exp(-10 n \eta)$, at most $(1 - 10\eta)$ points will be in the range $[0, q]$. We now apply the Bernstein inequality to the sum $\mb{E}_{[n]} [Z_i \bm{1} \lbrb{Z_i \leq q}] - \E[Z_i \bm{1} \lbrb{Z_i \leq q}]$ to control its lower tail; note that each element in the sum satisfies $\abs{Z_i \bm{1} \lbrb{Z_i \leq q} - \E[Z_i \bm{1} \lbrb{Z_i \leq q}]} \leq 2q \leq O(\frac{1}{\sqrt{\eta}})$ and similarly that $\var(Z_i \bm{1} \lbrb{Z_i \leq q}) \leq O(1)$.  An application of the Bernstein inequality shows with probability at least $1-\exp(-n/10)$, that $\mb{E}_{[n]} [Z_i \bm{1} \lbrb{Z_i \leq q}] \geq \E[Z_i \bm{1} \lbrb{Z_i \leq q}] - O(\sqrt{\eta}) \geq \E[Z_i]-O(\sqrt{\eta})$. Renormalizing the sum by the factor $\frac{n}{\abs{S}} = 1 + O(\eta)$, then shows $\frac{1}{\abs{S}} \sum_{i = 1}^n Z_i \bm{1} \lbrb{Z_i \leq q} \geq \E[Z_i]-O(\sqrt{\eta})$ since $\E[Z_i] \leq 1$. Now applying a union bound over the previous two events we obtain the lower bound for fixed $u_j$ -- since the at least $10 \eta$ points lying in the interval $(q, \infty)$ will be precisely the points which contribute the maximal values to the sum $\frac{1}{\abs{S}} \sum_{i = 1}^n Z_i$ -- so with probability at least $1-2\exp(-n \eta/10)$ for fixed $u_j$, the truncated sum contains the smallest $(1 - 10\eta)n$ of the $Z_i$, which is smaller then any arbitrary set of $(1 - 10\eta)n$ of the $Z_i$. Applying another union bound over the entire $\delta$-net shows that the state holds uniformly for any set of $(1-10\eta)n$ samples $S$, and any $u_j \in G$ that $\mb{E}_S [Z_i] \geq u_j^\top \Sigma u_j - O(\sqrt{\eta}$ with probability at least $1-2 (\frac{3}{\delta})^{d} \exp(-n/10 \eta)$. To lift this to all $v \in S^{d-1}$, define $Q=\mb{E}_S [X_i X_i^\top]$ and note that each element $v = u_j +\Delta$ for some $u_j \in G$ and $\Delta$ such that $\norm{\Delta} \leq \delta$. Thus,
\begin{align*}
     \langle v, Q v\rangle = \langle u_j, Q u_j \rangle  + 2 \langle v, Q \Delta \rangle +\langle \Delta, Q \Delta \rangle \geq \langle u_j, Q u_j \rangle - O(\delta d),
\end{align*}
uniformly for all $\Delta$ where we have used the crude bound that $\norm{Q} \leq O(d)$ since the vectors $X_i$ are truncated at $O(\sqrt{d})$. Choosing $\delta = \frac{c \sqrt{\eta}}{d^2}$ for sufficiently small $c$, then gives the final conclusion that $\frac{1}{\abs{S}} \sum_{i \in S}\langle v, Qv \rangle \geq \Sigma - O(\sqrt{\eta})$ for all $v \in S^{d-1}$ with probability at least $1-2 (c_1 d^2/\sqrt{\eta})^{d} \exp(-n/10 \eta)$. Choosing $n \geq c_2/\eta d \log(2d/\sqrt{\eta})$ for sufficiently large $c_2$ ensures this event holds with probability at least $1-\frac{1}{2d^2}$.
\end{proof}

We now prove \cref{lem:htcovconcset} which exchanges the order of the $\min$ and $\max$ in the SDP used in the gradient estimation step. 

\begin{proof}[Proof of \cref{lem:htcovconcset}]
    First, let $T = \{i: \abs{\inp{X_i}{u}} \leq 40 / \sqrt{\eta}\}$. We now consider the following min-max game:
    \begin{equation*}
        \min_{\alpha \in \Delta_{\eta / 25}} \max_{M \succcurlyeq 0, \tr{M} = 1} \inp*{\mb{E}_{\alpha} [\inp{X_i}{u}^2 X_iX_i^\top]}{M} = \max_{M \succcurlyeq 0, \tr{M} = 1} \min_{\alpha \in \Delta_{\eta / 25}} \inp*{\mb{E}_{\alpha} [\inp{X_i}{u}^2 X_iX_i^\top]}{M},
    \end{equation*}
    where the min and the max can be exchanged as a consequence of von Neumann's min-max theorem. Assume now that the conclusion of \cref{lem:htcovconc} holds. Then, we have that the right hand side of the above display is upper bounded by $\cfour$ and there exists $\alpha^* \in \Delta_{\eta / 25}$ such that:
    \begin{equation*}
        \mb{E}_{\alpha^*} [\inp{X_i}{u}^2 X_i X_i^\top] \preccurlyeq \cfour.
    \end{equation*}
    Now, let $S = \{i \in T: \alpha_i \geq 1/(2n)\}$. Note that $S$ being a subset of $T$ automatically proves the second claim of the lemma. We have from the fact that $\abs{T} \geq (1 - \eta / 25) n$:
    \begin{equation*}
        1 - \frac{\eta}{25 - \eta} = \sum_{i \in S} \alpha_i + \sum_{i \notin S} \alpha_i \leq \frac{\abs{T \setminus S}}{2n} + \frac{\abs{T} - \abs{T \setminus S}}{(1 - \eta / 25)n} \implies \abs{T \setminus S} \leq (\eta n) / 6,
    \end{equation*}
    where the inequality follows when $\eta$ is sufficiently small and the fact that $\abs{T} \leq n$. Finally, the statement of the lemma follows with $\cff = 3\cfour$:
    \begin{equation*}
        \mb{E}_{S} [\inp{X_i}{u}^2 X_iX_i^\top] \leq 3 \mb{E}_{\alpha^*} [\inp{X_i}{u}^2 X_iX_i^\top] \preccurlyeq 3 \cfour.
    \end{equation*}
\end{proof}

We now provide the proof of \cref{lem:htcovconcr1} which shows there exist a set of weights which can exclude problematic $X_i$ in the normalized sums we consider.

\begin{proof}[Proof of \cref{lem:htcovconcr1}]
    We first grid the unit spheres $\mc{G}_1 = \norm{u} \leq 1$ and $\mc{G}_2 = \norm{v} \leq 1$ with grids of resolution $1 / (100 \cp^4 d^2)$. That is, for all $\norm{u} \leq 1$, there exists $u' \in \mc{G}_1$ such that $\norm{u - u'} \leq 1 / (100\cp^4d^2)$ and similarly for $v$. Now, let $q$ be defined as follows:
    \begin{equation*}
        q = \min \lbrb{r: \forall \norm{u} = 1, \mb{P} \lprp{\abs{\inp{X_i}{u}} \leq r} \geq 1 - \eta / 400}.
    \end{equation*}
    We know that $q \leq 20 / \sqrt{\eta}$ by Markov's inequality applied to $\inp{X_i}{u}^2$. Now, for fixed $u$ and $v$, consider the random variables $Q_i = \bm{1} \lbrb{\abs{\inp{X_i}{u}} \leq 1000 / \sqrt{\eta}) \wedge \inp{X_i}{v} \leq 1000 / \sqrt{\eta}}$. By the union bound, we have $\bm{P}[Q_i = 1] \geq 1 - \eta / 200$. Furthermore, consider the random variables $W_i = \inp{X_i}{u}^2 \inp{X_i}{v}^2 Q_i$. We see that $W_i$ is bounded by $1000^2 / \eta$ and furthermore has variance bounded as follows:
    \begin{equation*}
        \mb{E} [W_i^2] \leq \frac{1000^2}{\eta} \mb{E} [\inp{X_i}{v}^4] \leq \frac{\hc 1000^2}{\eta} (v^\top \Sigma v)^2 \leq \frac{\hc 1000^2}{\eta}. 
    \end{equation*}
    Furthermore, its mean is bounded by Cauchy Schwarz and L4-L2 hypercontractivity as:
    \begin{equation*}
        \bm{E} [W_i] \leq \bm{E} [\inp{X_i}{u}^2 \inp{X_i}{v}^2] \leq \hc \cdot (u^\top \Sigma u) (v^\top \Sigma v) \leq \hc.
    \end{equation*}
    Therefore, applying Bernstein's inequality to the random variable $W = n^{-1} \sum_{i = 1}^n W_i$, we get:
    \begin{equation}
        \label{eq:htcovvconc}
        \bm{P} \lprp{W \geq 1 + \hc} \leq \exp \lprp{-\frac{n\eta}{2\cdot 1000^2(\hc + 1)}}.
    \end{equation}
    Finally, consider the random variable $K_i = \bm{1} \lbrb{\abs{\inp{X_i}{u}} \geq q \vee \abs{\inp{X_i}{v}} \geq q}$. We see that $K_i$ has variance and expectation bounded by $\eta / 200$ and is bounded by $1$. Therefore, applying Bernstein's inequality to $K = n^{-1} \sum_{i = 1} K_i$, we get that:
    \begin{equation}
        \label{eq:htcovnconc}
        \bm{P} \lprp{K \geq \frac{\eta}{100}} \leq \exp \lprp{- \frac{1}{200^2}\cdot \frac{\eta^2 n^2}{2 (\eta n / 200 + \eta n / 200)}} = \exp \lprp{- \frac{\eta n}{800}}.
    \end{equation}
    From standard packing/covering number bounds \citep[see, e.g.,][]{vershynin2012introduction}, we see that the number of elements in each grid is at most $\lprp{C d}^{2d}$ for some constant $C$. Therefore, the total number of pairs of elements, $(u,v)$, from each grid is at most $\lprp{C^2 d}^{4d}$. Therefore, taking $n = \tO (d / \eta)$, we have that for all $u \in \mc{G}_1$ and $v \in \mc{G}_2$ the events described in \cref{eq:htcovvconc,eq:htcovnconc} hold with probability at least $1 - 1 /d$. Now, consider any $\norm{u} \leq 1$ and $\norm{v} \leq 1$. Let $u' \in \mc{G}_1$ and $v' \in \mc{G}_2$ be such that $\norm{u' - u} \leq 1 / (100 \cp^4 d^2)$ and $\norm{v' - v} \leq 1 / (100 \cp^4 d^2)$. Now, we have for any $X_i$:
    \begin{align*}
        &\inp{X_i}{u}^2 \inp{X_i}{v}^2 - \inp{X_i}{u'}^2 \inp{X_i}{v'}^2 = (\inp{X_i}{u}^2 \inp{X_i}{v}^2 - \inp{X_i}{u'}\inp{X_i}{u} \inp{X_i}{v}^2) + \\
        &\qquad (\inp{X_i}{u'}\inp{X_i}{u} \inp{X_i}{v}^2 - \inp{X_i}{u'}^2 \inp{X_i}{v}^2) + (\inp{X_i}{u'}^2\inp{X_i}{v}^2 - \inp{X_i}{u'}^2 \inp{X_i}{v'}\inp{X_i}{v}) + \\
        &\qquad (\inp{X_i}{u'}^2\inp{X_i}{v'}\inp{X_i}{v} - \inp{X_i}{u'}^2 \inp{X_i}{v'}^2).
    \end{align*}
    Now, each of the four terms is bounded by $\max(\norm{u' - u},\norm{v' - v}) \norm{X_i}^4$ using the Cauchy Schwarz inequality. From our bounds on $\norm{u' - u}$, $\norm{v' - v}$
    and $\norm{X_i}$, we obtain for all $i$:
    \begin{equation}
        \label{eq:htcovtbnd}
        \abs{\inp{X_i}{u}^2 \inp{X_i}{v}^2 - \inp{X_i}{u'}^2 \inp{X_i}{v'}^2} \leq 1.
    \end{equation}
    We also have, using Cauchy Schwarz inequality:
    \begin{equation}
        \label{eq:htcovipbnd}
        \abs{\inp{X_i}{u} - \inp{X_i}{u'}} \leq \norm{u - u'}\norm{X_i} \leq 1,
    \end{equation}
    with a similar bound holding for $v$.  From \cref{eq:htcovtbnd,eq:htcovipbnd}, we have:
    \begin{align*}
        &n^{-1} \sum_{i = 1}^n \inp{X_i}{u}^2 \inp{X_i}{v}^2 \bm{1} \lbrb{\abs{\inp{X_i}{u}} \leq 800 / \sqrt{\eta} \wedge \abs{\inp{X_i}{v}} \leq 800 / \sqrt{\eta}} \\
        &\leq 1 + n^{-1} \sum_{i = 1}^n \inp{X_i}{u'}^2 \inp{X_i}{v'}^2 \bm{1} \lbrb{\abs{\inp{X_i}{u}} \leq 800 / \sqrt{\eta} \wedge \abs{\inp{X_i}{v}} \leq 800 / \sqrt{\eta}} \\
        &\leq 1 + n^{-1} \sum_{i = 1}^n \inp{X_i}{u'}^2 \inp{X_i}{v'}^2 \bm{1} \lbrb{\abs{\inp{X_i}{u'}} \leq 1000 / \sqrt{\eta} \wedge \abs{\inp{X_i}{v'}} \leq 1000 / \sqrt{\eta}} \\
        &\leq \hc + 2.
    \end{align*}
    For the second statement of the lemma, we have from \cref{eq:htcovipbnd}:
    \begin{equation*}
        \sum_{i = 1}^n \bm{1} \lbrb{\abs{\inp{X_i}{u}} \leq 40/\sqrt{\eta} \wedge \abs{\inp{X_i}{v}} \leq 40 / \sqrt{\eta}} \geq \sum_{i = 1}^n \bm{1} \lbrb{\abs{\inp{X_i}{u'}} \leq q \wedge \abs{\inp{X_i}{v'}} \leq q} \geq n (1 - \eta / 100).
    \end{equation*}
 This concludes the proof of the lemma.
\end{proof}

We now present the proof of \cref{lem:htcovconc} which allows us to lift the previous argument which was restricted to rank-one $M$ to general $M$.

\begin{proof}[Proof of \cref{lem:htcovconc}]
    We first condition on the event which holds in the conclusion of Lemma~\ref{lem:htcovconcr1} and consider a vector $u$ of norm  $\norm{u} = 1$ and a trace-1 psd matrix $M$. Let $g \thicksim \mc{N} (0, M)$ be a centered gaussian random vector with covariance matrix $M$. We have by an application of the Borell-TIS inequality~~\citep[see, e.g.,][]{ledoux1991probability} that with probability at least $1 - \exp (-9/2) \geq 0.95$, 
     $\norm{g} \leq 4$. Furthermore, for any $i$, we have that $\inp{X_i}{g}$ is gaussian with mean $0$ and variance $\inp{X_iX_i^\top}{M}$. Therefore, we get:
    
    \begin{equation*}
      \mb{P} \Paren{\abs{\inp{g}{X_i}} \geq 0.5 \sqrt{\inp{X_iX_i^\top}{M}}} \geq 1 - 1 / \sqrt{2\pi} \geq 0.6.  
    \end{equation*}
    So, for any $i$, we have:
    \begin{equation*}
        \mb{P} \Paren{\abs{\inp{g}{X_i}} \geq 0.5 \sqrt{\inp{X_iX_i^\top}{M}} \wedge \norm{g} \leq 4} \geq 0.5.
    \end{equation*}
    Let $\mc{I}$ denote the indices such that $\sqrt{\inp{X_iX_i^\top}{M}} \geq 400 / \sqrt{\eta}$. Then, we have:
    \begin{equation*}
        \mb{E} \lsrs{\sum_{i \in \mc{I}} \bm{1} \lbrb{\abs{\inp{g}{X_i}} \geq 200 / \sqrt{\eta} \wedge \norm{g} \leq 4}} = \mb{E} \lsrs{\sum_{i \in \mc{I}} \bm{1} \lbrb{\abs*{\inp*{\frac{g}{4}}{X_i}} \geq 50 / \sqrt{\eta} \wedge \norm{g} \leq 4}} \geq \frac{1}{2}  \abs{\mc{I}}.
    \end{equation*}
    From the second conclusion of \cref{lem:htcovconcr1}, each term inside the expectation is bounded by $\frac{\eta}{100}\cdot n$. Therefore, we have $\abs{\mc{I}} \leq \frac{\eta n}{50}$. This establishes the second claim of the lemma as there are at most $\frac{\eta n}{100}$ points, $X_i$, with $\abs{\inp{X_i}{u}} \geq 40 / \eta$.
    
    Now, for $j \in [n] \setminus \mc{I}$, we have:
    \begin{align*}
        &\mb{E} [\inp{X_j}{g}^2 \bm{1} \lbrb{\norm{g} \leq 4 \wedge \abs{\inp{X_j}{g}} \leq 800 / \sqrt{\eta}}] \geq (\mb{E} [\inp{X_j}{g} \bm{1} \lbrb{\norm{g} \leq 4 \wedge \abs{\inp{X_j}{g}} \leq 800 / \sqrt{\eta}}])^2 \\
        &\qquad \geq \inp{X_jX_j^\top}{M} \cdot \frac{1}{2\pi} \cdot \lprp{\int_{-1}^1 \abs{x}\exp \lprp{- \frac{x^2}{2}} dx}^2 \geq \frac{1}{4\pi} \cdot \inp{X_jX_j^\top}{M},
    \end{align*}
    this is because the event $\{\norm{g} \leq 4 \wedge \abs{\inp{X_j}{g}} \leq 800 / \sqrt{\eta}\}$ happens with probability at least $0.9$ and it can be seen via a simple monotonicity argument that the value of $\mb{E} [Z^2 \bm{1} \lbrb{A}]$ for $Z \thicksim \mc{N}(0,1)$ is minimized when $A$ is chosen as an interval around the origin with probability mass $\mb{P} (A)$. 
    
    From the previous display, we finally get:
    \begin{equation*}
        16 \ccov \geq 16 \cdot \mb{E} \lsrs{\sum_{j \notin \mc{I}} \inp*{\frac{gg^\top}{16}}{X_jX_j^\top} \bm{1} \lbrb{\norm{g} \leq 4 \wedge \abs{\inp{X_j}{g}} \leq 800 / \sqrt{\eta}}} \geq \frac{1}{4\pi} \sum_{j \notin \mc{I}} \inp{X_j X_j^\top}{M}.
    \end{equation*}
    By rearranging the above display, we get the first claim of the lemma with $\cfour = 64\pi \ccov$.
\end{proof}

Here we state and prove the result needed to show there is a good set of weights to remove errors from terms of the form $\epsilon_i X_i$ in the gradient term.

\begin{lemma}
    \label{lem:hterrconc}
    Let $\mc{D}$ satisfy \cref{as:htassump}. Then, for $n = \tO (d / \eta)$, we have with probability $99 / 100$, there exists a set $S \subset [n]$ with $\abs{S} \geq 1 - \eta / 100$ such that for any set $T \subset S$ with $\abs{T} \geq (1 - 10 \eta)$, we have:
    \begin{equation*}
        \norm{\abs{T}^{-1} \sum_{i \in T} \epsilon_i X_i} \leq 15\sqrt{\eta} \sigma \text{ and } \abs{T}^{-1} \sum_{i \in T} \epsilon_i^2 X_iX_i^\top \preccurlyeq 5 \sigma^2 I.
    \end{equation*}

    And furthermore, we have $\norm{\epsilon_i X_i} \leq O(\sigma \sqrt{d / \eta})$ for all $i \in S$.
\end{lemma}

\begin{proof}[Proof of \cref{lem:hterrconc}]
    We first note that $Z_i = \epsilon_i X_i$ is a random vector with zero mean and variance $\sigma^2 \Sigma$. Therefore, we have that $\mb{P} (\norm{Z_i} \leq 10 \sigma \sqrt{d / \eta}) \geq 1 - \eta/200$ by Markov's Inequality applied to $\norm{Z_i}^2$. Now, our set $S$ will be defined as follows $S = \{i: \norm{Z_i} \leq 10 \sigma \sqrt{d / \eta} \}$. Let $W_i = \bm{1} \{\norm{Z_i} \leq 10 \sigma \sqrt{d / \eta}\}$. We have $P(W_i = 1) \geq 1 - \eta / 100$ and $W_i$ has variance bounded by $\eta / 100$. Therefore, we get by Bernstein's inequality:
    \begin{equation*}
        \bm{P} \lprp{\sum_{i = 1}^n W_i \leq 1 - \eta / 50} \leq \exp \lbrb{- \frac{1}{100^2}\cdot \frac{\eta^2 n^2}{2 \lprp{n \cdot \eta / 100 + n \cdot \eta / 100}}} = \exp \lbrb{- \frac{\eta n}{400}}.
    \end{equation*}
    Therefore, we have with probability at least $1 - 1 / e^d$. Let $A = \{Z: \norm{Z} \leq 10 \sigma \sqrt{d / \eta}\}$. Now, conditioned on $Z_i \in A$, the conditional mean and second moment matrix of $Z_i$ are defined as follows for $Z_k \sim \mc{D}$:
    \begin{equation*}
        \mu_A = \frac{1}{\mb{P} (Z_k \in A)} \mb{E} [Z_k \bm{1} \lbrb{Z_k \in A}] \text{ and } \frac{1}{\mb{P} (Z_k \in A)} \mb{E} [Z_kZ_k^\top \bm{1} \lbrb{Z_k \in A}],
    \end{equation*}
    and furthermore, all the vectors in $S$ are independent of each other. We will now bound the mean term in the above display. Take $\norm{u} = 1$:
    \begin{align*}
        \frac{1}{\mb{P} (Z_k \in A)} \mb{E}[\inp{Z_k}{u} \bm{1} \lbrb{Z_k \in A}] &= \frac{1}{\mb{P} (Z_k \in A)} \mb{E}[\inp{Z_k}{u} (1 - \bm{1} \lbrb{Z_k \notin A})] \\
        &= \frac{1}{\mb{P} (Z_k \in A)} \mb{E}[\inp{Z_k}{u} \bm{1} \lbrb{Z_k \notin A}] \\
        &\leq \frac{1}{\mb{P} (Z_k \in A)} (\mb{E}[\inp{Z_k}{u}^2])^{1/2} (\mb{E} [\bm{1} \lbrb{Z_k \notin A}])^{1/2} \leq 2 \sqrt{\eta} \sigma,
    \end{align*}
    where the first inequality follows from Cauchy-Schwarz. Therefore, the mean of $Z_i$ for $Z_i \in S$ is bounded by $2\sqrt{\eta}$. For the variance, we have:
    \begin{equation*}
        \frac{1}{\mb{P} (Z_k \in A)} \mb{E} [Z_k Z_k^\top \bm{1} \{Z_k \in A\}] \preccurlyeq 2 \cdot \mb{E} [Z_kZ_k^\top] = 2 \sigma^2 \Sigma.
    \end{equation*}
    Therefore, we have with probability at least $0.999$ that $\abs{S}^{-1} \sum_{i \in S} Z_i \leq 4 \sigma \sqrt{\eta}$. Furthermore, we have by matrix Bernstein \cite{tropp2015introduction} that:
    \begin{equation*}
        \mb{P} \lprp{\frac{1}{\abs{S}} \sum_{i \in S} Z_i Z_i^\top \preccurlyeq 4 \sigma^2 \cdot I} \geq 1 - 2d \exp \lprp{- \frac{1}{2} \cdot \frac{n^2}{\hc (4n + 40 n d / \eta)}} \geq 1 - \frac{1}{d^2},
    \end{equation*}
    where we have used $(\mb{P}(Z_k \in A))^{-1} \mb{E} [\norm{Z_k}^2 Z_k Z_k^\top] \leq 2 \mb{E} [\norm{Z_k}^2 Z_k Z_k^\top]$ and \cref{lem:l4l2norm} to bound the matrix variance term. Therefore, we have with probability at least $0.99$ that there exists a set $S \subset [n]$ with $\abs{S} \geq (1 - \eta/100$)n such that:
    \begin{equation*}
        \mb{E}_S [Z_i] \leq 4 \sigma \sqrt{\eta} \text{ and } \mb{E}_S [Z_iZ_i^\top] \preccurlyeq 4\sigma^2 I.
    \end{equation*}
    Finally, we have for any $T \subset S$ with $\abs{T} \geq (1 - 10\eta)n$ and any $\norm{u} = 1$:
    \begin{align*}
        \mb{E}_T [\inp{u}{Z_i}] &= \frac{1}{\mb{P}_S (Z_i \in T)} \mb{E}_S [\inp{u}{Z_i} (1 - \bm{1} \lbrb{i \notin T})] \\
        &\leq 5 \sigma \sqrt{\eta} + 1.01 (\mb{E}_S [\inp{u}{Z_i}^2])^{1/2}(\mb{E}_S [\bm{1} \lbrb{Z_i \notin T}])^{1/2} \leq 15 \sigma \sqrt{\eta}.
    \end{align*}
    The upper bound on the second moment follows from the fact that $\mb{E}_T [Z_i Z_i^\top] \preccurlyeq \frac{1}{\mb{P}_S (Z_i \in T)} \mb{E} [Z_iZ_i^\top]$.
\end{proof}

We can now assemble the previous results to provide the proof of the final lemma which establishes the deterministic conditions:

\begin{proof}[Proof of \cref{lem:htdet}]
    Under the assumptions, the conclusions of \cref{lem:htcovconcset,lem:hterrconc,lem:htmeanconc} all hold with probability $0.99$. Let $u = \frac{w - w^*}{\norm{w - w^*}}$ and label the set guaranteed to exist by \cref{lem:htcovconcset} be $S_1$ and the one by \cref{lem:hterrconc} by $S_2$. And consider the intersection of the two sets $S_3 = S_1 \cap S_2$. By the union bound, $\abs{S_3} \geq 1 - \eta / 5$. Since, at most $2\eta n$ of the elements of $S_3$ are corrupted, there exists a set $S \subset S_3$ of size at least $(1 - 5\eta)n$ such that none of their elements have been corrupted. Now, we have for $i \in S$:
    \begin{equation*}
        \norm{G_i} = \norm*{X_i \inp{X_i}{w - w^*} - \epsilon_i X_i} \leq \norm*{X_i \inp{X_i}{w - w^*}} + \norm{\epsilon_i X_i} \leq O \lprp{\sqrt{\frac{d}{\eta}} \norm{w - w^*} + \sigma \sqrt{\frac{d}{\eta}}}.
    \end{equation*}
    This proves the third claim of the lemma. For the first claim, we have:
    \begin{multline*}
        \norm{\abs{S}^{-1} \sum_{i \in S} G_i(w) - \Sigma (w - w^*)} \leq \norm{(\abs{S}^{-1} X_iX_i^\top - \Sigma) (w - w^*)} + \norm{\abs{S}^{-1} \sum_{i \in S} \epsilon_i X_i}  \\
        \leq \norm{\abs{S}^{-1} X_iX_i^\top - \Sigma} \norm{w - w^*} + O(\sqrt{\eta} \sigma) \leq O(\sqrt{\eta} (\norm{w - w^*} + \sigma)).
    \end{multline*}
    Finally, for the second claim of the lemma, we have:
    \begin{equation*}
        \abs{S}^{-1} \sum_{i \in S} G_i(w)G_i(w)^\top \preccurlyeq 2 \abs{S}^{-1} \sum_{i \in S} \inp{X_i}{w - w^*}^2 X_iX_i^\top + \epsilon_i^2 X_iX_i^\top \preccurlyeq O(\norm{w - w^*}^2 + \sigma^2) \cdot I.
    \end{equation*}
\end{proof}
With this result in hand we can establish the deterministic conditions necessary for our algorithms hold with the requisite probability.
\section{Sub-Gaussian Concentration}
\label{ap:sgconc}

In this section we prove concentration results pertaining to the case where the random variables, $X$ and $\epsilon$ are sub-Gaussian instead of the much milder L4-L2 assumptions we had previously considered. This allows us to obtain much tighter bounds on our recovery error than those obtained in \cref{sec:htconc}.

From \cref{thm:htmain}, we can assume that we have an initial estimate $w$ such that $(w - w^*) \leq O(\sqrt{\eta})$. Therefore, by subtracting out $\inp{X_i}{w}$ from all the data points, our problem reduces to a setting where we additionally have $\norm{w^*} \leq O(\sqrt{\eta})$. We make this formal in the following assumption:
\begin{assumption}
    \label{as:sgiestbis}
    We assume that there exists a constant $\nu$, such that $\norm{w^*} \leq \nu \sigma \sqrt{\eta}$.
\end{assumption}
Note, that under \cref{as:sgassump}, we have $\mb{E} [\norm{X}] \leq \lprp{\mb{E} [\norm{X}^2]}^{1/2} = \sqrt{d}$. Throughout the analysis we refer at several points to the following events:
\begin{gather*}
    \mc{H}_{\epsilon} = \{\abs{\epsilon_i} \leq 6 \phi \sqrt{\log 1 / \eta}\} \\ 
    \mc{H}_{l} = \{\norm{X} \leq 7\psi \sqrt{d} + 24 \psi \sqrt{\log 1 / \eta}\} \\
    \mc{H}_{o} = \{\abs{\inp{X}{w - w^*}} \leq 6 \psi \norm{w - w^*} \sqrt{\log 1 / \eta}\}.
\end{gather*}
Note, that from \cref{lem:sgmeanconc} the event, $\mc{H} = \mc{H}_{\epsilon} \cap \mc{H}_l \cap \mc{H}_o$ occurs with probability at least $1 - \eta^{16}$. We start by proving a lemma pertaining to the noise terms arising in the gradient.
\begin{lemma}
    \label{lem:sgerrconcapp}
    Let $\mc{D}$ satisfy \cref{as:sgassump}. Then for $n = \tO (d / \eta^2)$samples from $\mc{D}$, we have with probability at least $1 - 3 / (10d^2)$, for any set $S \subset [n]$ with $\abs{S} \geq (1 - 10\eta)n$, we have:
    \begin{equation*}
       \norm*{ \sum_{i \in S} \frac{Z_i}{n}} \leq O(\sigma \eta \log 1 / \eta) \text{ and } \norm*{\sum_{i \in S} \frac{Z_iZ_i^\top}{n} - \sigma^2 I} \leq O(\sigma^2 \eta \log^2 1 / \eta),
    \end{equation*}
     where $Z_i = \epsilon_i X_i\bm{1} \lbrb{(\epsilon_i, X_i) \in \mc{H}}$.
\end{lemma}
\begin{proof}[Proof of \cref{lem:sgerrconcapp}]
Consider the random matrix $M = n^{-1}\sum_{i = 1}^n M_i$ where the matrices $M_i$ are defined by $M_i = \epsilon_i^2 X_i X_i^\top \bm{1} \{(X_i, \epsilon_i) \in \mc{H}\}$. From the definition of the event $\mc H$, we can bound each of them as $M_i \preccurlyeq C \phi^2 \psi^2 \log 1 / \eta (d + \log 1 / \eta)$. Furthermore, we have $\mb{E} [M_i] \preccurlyeq \mb{E} [\epsilon_i^2 X_iX_i^\top] = \sigma^2 \cdot I$. For a lower bound, we have for any unit vector $u$, by Cauchy-Schwarz inequality and standard bounds on the moments of sub-Gaussian distributions:
\begin{equation*}
    \mb{E} [u^\top M_i u] = \mb{E} [\epsilon_i^2 \inp{X_i}{u}^2 (1 - \bm{1}\lbrb{X_i \notin \mc{H}})] \geq \sigma^2 - (\mb{E} [\epsilon_i^4] \mb{E} [\inp{X_i}{u}^4])^{1/2}(1 - \mb{P}(\mc{H}))^{1/2} \geq \sigma^2 (1 - \eta^8).
\end{equation*}
Therefore, applying the matrix Bernstein inequality \citep[see, e.g.,][]{tropp2015introduction} leads to:
\begin{equation*}
    \mb{P} \lprp{\norm{M - \sigma^2 \cdot I} \geq \sigma^2 \eta} \leq 2d \exp \lprp{- \frac{\eta^2 n^2}{C'(nd + n\eta \log 1 / \eta (d + \log 1 / \eta))}} \leq 1 - \frac{1}{10d^2},
\end{equation*}
where we have used the fact that sub-Gaussian random variables are L4-L2 hypercontractive (\cref{lem:sgmoms}), \cref{lem:l4l2norm}, \cref{as:sgassump} and our setting of $n$.

For the lower bound, we begin by gridding the unit sphere with a grid of resolution $\lprp{\frac{\eta}{d}}^{8}$. From standard results, there exists a grid, $\mc{G}$ of size $\lprp{C^\dagger \frac{d}{\eta}}^{8d}$ for some absolute constant $C^\dagger$. Now, for a specific $u \in \mc{G}$, consider the random variable $W_i = \epsilon_i^2 \inp{X_i}{u}^2 \bm{1} \lbrb{\mc{H} \cap \mc{H}_u}$ where $\mc{H}_u = \{\abs{\inp{X_i}{u}} \leq 6 \psi \sqrt{\log 1 / \eta}\}$. We see that $P(\mc{H}_u) \geq 1 - \eta^{16}$. Therefore, we have by a similar argument as before that $\mb{P} (\mc{H} \cap \mc{H}_u) \geq 1 - 2\eta^8$. Furthermore, note that $W_i$ is upper bounded by $36^2 \psi^2 \phi^2 \log^2 1 / \eta$. Therefore, applying Hoeffding's inequality to $W = n^{-1}\sum_{i = 1}^n W_i$, we get:
\begin{equation*}
    \mb{P} (W \leq \sigma^2(1 - \eta)) \leq \exp \lprp{- \frac{n \eta^2}{(36^2 \psi^2 \phi^2 \log^2 1 / \eta)^2}}.
\end{equation*}
By taking a union bound over $\mc{G}$, we see that the above concentration holds for all $u \in \mc{G}$ with probability at least $1 - 1 / (10d^2)$. Now, for a non-grid element $v$, select $v' \in \mc{G}$ such that $\norm{v - v'} \leq (\frac{\eta}{d})^{8}$. Note, that for any $(X_i, \epsilon_i) \in \mc{H}$, we have that $\abs{\inp{X_i}{v} - \inp{X_i}{v'}} \leq \epsilon^4$ and $\abs{\epsilon_i^2 \inp{X_i}{v}^2 - \epsilon_i^2 \inp{X_i}{v}^2} \leq \epsilon^4$. Therefore, we conclude that for all $\norm{u} = 1$, we have:
\begin{equation*}
    \sum_{i = 1}^n \epsilon_i^2 \inp{X_i}{u}^2 \bm{1} \lbrb{(X_i, \epsilon_i) \in \mc{H} \wedge \abs{\inp{X_i}{u}} \leq 7 \psi \sqrt{\log 1 / \eta}} \geq 1 - 2\eta.
\end{equation*}
Therefore, the removal of $10\eta n$ points at most distorts $W$ by a factor of $10\eta \cdot 36^2 \phi^2\psi^2 \log^2 1 / \eta$, thus proving the second claim of the lemma.

For the first claim, consider the random vector $Z_i = \epsilon_i X_i \bm{1} \lbrb{(X_i, \epsilon_i) \in \mc{H}}$. To bound the mean of $Z_i$, we consider as before a unit vector $u$:
\begin{equation*}
    \mb{E}[\inp{u}{Z_i} \bm{1} \lbrb{(X_i, \epsilon_i) \in \mc{H}}] \leq (\mb{E} [\inp{u}{Z_i}^2])^{1/2} (\mb{P} (\mc{H}))^{1/2} \leq \eta^8 \sigma.
\end{equation*}
Since, $u$ is arbitrary, we have that the mean of the norm of $Z_i$ is bounded by $\eta^8 \sigma$. Also, note from the definition of $Z_i$ that $Z_i$ is sub-Gaussian with sub-Gaussian parameter at most $6 \phi \psi \sqrt{\log 1 / \eta}$. Therefore, we have with probability at most $1 - 1 / (10d^2)$ from \cref{lem:lensg} that $\norm{n^{-1} \sum_{i = 1}^n Z_i} \leq \phi \psi \eta$. Now, we have for any $T \subset [n]$ with $\abs{T} \geq (1 - 10\eta)n$ and any $\norm{u} = 1$:
\begin{align*}
    \inp*{u}{\mb{E}_{[n]} \lsrs{Z_i} - \mb{E}_{[n]} \lsrs{Z_i \bm{1} \lbrb{i \in T}}} &= \mb{E}_{[n]} \lsrs{\inp{u}{Z_i} \bm{1} \lbrb{i \notin T}} \\
    &\leq (\mb{E}_{n} \lsrs{\inp{u}{Z_i}^2 \bm{1} \lbrb{i \notin T}})^{1/2} (\mb{E}_{[n]} \lsrs{\bm{1} \lbrb{i \notin T}})^{1/2} \\
    &\leq (\sigma^2 \eta + 10 \eta \cdot 36^2 \phi^2\psi^2 \log^2 1 / \eta)^{1/2} \sqrt{10 \eta}\\
    &\leq 20 \phi \psi \eta \log 1 / \eta,
\end{align*}
where the first inequality follows from Cauchy-Schwarz and the second inequality follows from the second claim of the lemma. This concludes the proof of the lemma.
\end{proof}

The following lemma is the analogue \cref{lem:htmeanconc} for the sub-Gaussian case. It's proof follows similarly to \cref{lem:htmeanconc}.
\begin{lemma}
    \label{lem:sgmeanconc}
    Let $\{X_i\}_{i = 1}^n$ satisfy  \cref{as:sgassump}. Then, there exists a universal constant $c$ such that  if $n = \tO (d / \eta^2)$, with probability at least $1 - 1 / (10d^2)$ for any set of $(1 - 10\eta) n$ samples, $S$, we have:
    \begin{equation*}
        (1 - O(\eta \log 1 / \eta)) \Sigma \preccurlyeq \mb{E}_S [X_i X_i^\top \bm{1} \{(X_i, \epsilon_i) \in \mc{H}\}] \preccurlyeq (1 + O(\eta)) \Sigma. 
    \end{equation*}
\end{lemma}
\begin{proof}[Proof of \cref{lem:sgmeanconc}]
    For the upper bound, consider the random matrix, $Z_i = X_i X_i^\top \bm{1} \lbrb{(X_i, \epsilon_i) \in \mc{H}}$. We first notice that $\norm{Z_i} = \norm{X_i}^2 \leq O(d + \log 1 / \eta)$ from the definition of $\mc{H}$. Furthermore, we have from \cref{lem:l4l2norm} and \cref{lem:sgmoms}, that $\norm{\mb{E} [(Z_i - \mb{E} Z_i)^2]} \leq \norm{\mb{E} [Z_i^2]} \leq O(d)$. Therefore, we may apply the matrix Bernstein inequality to $Z = \mb{E}_{[n]} [Z_i]$ to obtain:
    \begin{equation*}
        \mb{P} \lprp{\norm{Z - \mb{E} [Z]} \geq \eta} \leq d \exp \lprp{- \frac{n^2 \eta^2}{C (dn + (d + \log 1 / \eta) n \eta)}} \leq \frac{1}{100d^2},
    \end{equation*}
    for our value of $n$. To bound $\mb{E}[Z_i]$, we have for all $u$:
    \begin{equation*}
        \abs{1 - \mb{E} [u^\top Z_i u]} = \abs{\mb{E} [\inp{X_i}{u}^2 \bm{1} \lprp{(X_i, \epsilon_i) \notin \mc{H}}]} \leq (\mb{E} [\inp{X_i}{u}^4])^{1 / 2} (\mb{P} ((X_i, \epsilon_i) \notin \mc{H}))^{1 / 2} \leq \eta^6,
    \end{equation*}
    from \cref{lem:sgmoms} and the definition of $\mc{H}$. The upper bound now follows from the fact that $\mb{E}_S [Z_i] \preccurlyeq (1 - 10 \eta)^{-1} \mb{E}_{[n]} [Z_i]$.
    
    For the lower bound, we begin by constructing a $\delta$-net of the unit sphere of resolution $\lprp{\frac{\eta}{d}}^8$. As in previous lemmas, there exists, a grid $\mc{G}$, of size $\lprp{C' \frac{d}{\eta}}^{8d}$. Now, for a particular $u \in \mc{G}$, consider the random variable $W_i = \inp{X_i}{u}^2 \bm{1} \lbrb{\mc{H} \cap \mc{H}_u}$ where $\mc{H}_u = \bm{1} \lbrb{\abs{\inp{X_i}{u}} \leq 6 \psi \sqrt{\log 1 / \eta}}$. We have from \cref{lem:sgtrunc} that:
    \begin{equation*}
        \abs{1 - \mb{E}[W_i]} \leq \abs{\mb{E}[\inp{X_i}{u}^2 \bm{1} \lbrb{(X_i, \epsilon_i) \notin \mc{H} \cap \mc{H}_u}]} \leq (\mb{E} [\inp{X_i}{u}^4])^{1 / 2} (\mb{P} ((X_i, \epsilon_i) \notin \mc{H} \cap \mc{H}_u))^{1 / 2} \leq \eta^4,
    \end{equation*}
    from \cref{lem:sgmoms} and the definition of $\mc{H}$ and $\mc{H}_u$. Now, applying Hoeffding's inequality to $W = \mb{E}_{[n]} [W_i]$, we get:
    \begin{equation*}
        \mb{P} \lprp{\abs{W - \mb{E}[W]} \geq \eta} \leq \exp \lprp{- \frac{2 n^2 \eta^2}{36 \psi^2 \log 1 / \eta}}.
    \end{equation*}
    Therefore, the above event holds uniformly over all $u \in \mc{G}$ with probability at least $1 - 1 / (100d^2)$ for our value of $n$. For a non grid point, $u$, consider $v \in \mc{G}$ closest to $u$, now we have for all $i$ satisfying $(X_i, Y_i) \in \mc{H}$ that: 
    \begin{equation*}
        \abs{\inp{X_i}{u}^2 - \inp{X_i}{v}^2} \leq 2\abs{\inp{X_i}{u - v}\inp{X_i}{u - v}} + \abs{\inp{X_i}{u - v}^2} \leq 3 \norm{X_i}^2 \norm{u - v} \leq \eta^4.
    \end{equation*}
    Therefore, we have for all $\norm{u} = 1$:
    \begin{equation*}
        \mb{E}_{[n]} [\inp{X_i}{u}^2 \bm{1} \lbrb{(X_i, \epsilon_i) \in \mc{H} \wedge {\abs{\inp{X_i}{u}} \leq 50 \psi^2 \log 1 / \eta }}] \geq 1 - 2 \eta.
    \end{equation*}
    The lower bound now follows from the fact that removing $10 \eta$ points at most reduces the above sum by $500 \eta \psi^2 \log 1 / \eta$ and the fact that $\mb{E}_{S} [W_i] \geq \mb{E}_{[n] [W_i \bm{1} \lbrb{i \in S}]}$.
\end{proof}

In the following lemma, we prove that there exists a good set of data points whose gradients have well behaved covariance structure.

\begin{lemma}
    \label{lem:sgcovconcapp}
    Let $\mc{D}$ satisfy \cref{as:sgassump,as:sgiest}. Then for $n = \tO (d / \eta^2)$, we have with probability at least $1 - 1 / (2d^2)$, that for any $S \subset [n]$ with $\abs{S} \geq (1 - 10\eta)n$:
    \begin{equation*}
        \norm*{n^{-1} \sum_{i \in S} G_i(w^*) G_i(w^*)^\top \bm{1} \lbrb{(X_i, \epsilon_i) \in \mc{H}} - \sigma^2 \cdot I} \leq O(\eta \log^2 1 / \eta))\sigma^2.
    \end{equation*}
\end{lemma}
\begin{proof}[Proof of \cref{lem:sgcovconcapp}]
    We first explicitly write out the expression for $G_i(w^*) G_i(w^*)^\top$ as follows:
    \begin{equation*}
        G_i(w^*) G_i(w^*)^\top = \inp{X_i}{w^*}^2 X_iX_i^\top + 2\epsilon_i \inp{X_i}{w^*} X_i X_i^\top + \epsilon_i^2 X_iX_i^\top.
    \end{equation*}
    For the first term in the sum, consider the random matrix $Z = n^{-1} \sum_{i = 1}^n Z_i$ where $Z_i = \inp{X_i}{w^*}^2 X_i X_i^\top \bm{1}\lbrb{(X_i, \epsilon_i) \in \mc{H}}$. We first note that:
    \begin{multline*}
        \mb{E} [Z_i^2] = \mb{E} [\inp{X_i}{w^*}^4 \norm{X_i}^2 X_i X_i^\top \bm{1} \lbrb{(X_i, \epsilon_i) \in \mc{H}}] \\
        \preccurlyeq O(\eta^2 \sigma^4 \log^2 1 / \eta) \mb{E} [\norm{X_i}^2 X_i X_i^\top \bm{1} \lbrb{(X_i, \epsilon_i) \in \mc{H}}] \preccurlyeq O(\sigma^4 d \eta^2 \log^2 1 / \eta) \cdot I,
    \end{multline*}
     where the first inequality follows from the definition of $\mc{H}$ and the last inequality follows from the hyper-contractivity of sub-Gaussian distributions and \cref{lem:l4l2norm}. Furthermore, we have $\norm{Z_i} \leq O(\sigma^2 \eta \log 1 / \eta (d + \log 1 / \eta))$. To bound the expectation of $Z_i$, we have for all $\norm{u} = 1$:
    \begin{equation*}
        \mb{E} [u^\top Z_iu] \leq \mb{E} [\inp{X_i}{w^*}^2 \inp{X_i}{u}^2] \leq O(\sigma^2 \eta),
    \end{equation*}
    where the last inequality follows from the hypercontractivity of sub-Gaussian distributions and Cauchy-Schwarz. There, we have $\norm{\mb{E}[Z_i]} \leq O(\sigma^2 \eta)$. Now, applying matrix-Bernstein to $Z$, we get that:
    \begin{equation*}
        \mb{P} \lprp{\norm{Z - \mb{E}[Z_i]} \geq \sigma^2 \eta} \leq 2d \exp \lprp{- \frac{n^2\sigma^4 \eta^2}{C(n\sigma^4 d \eta^2 \log^2 1 / \eta + n\sigma^4 \eta^2)}} \leq \frac{1}{10d^2},
    \end{equation*}
    where the last inequality follows from our setting of $n$. We have now established control over the first term in the expansion of $G_iG_i^\top$. 
    
    To control the last term, we will condition on the event from \cref{lem:sgerrconcapp} which occurs with probability at least $1 - 3 / (10d^2)$.
    
    For the middle term, we first define the random matrix $W_i = \epsilon_i \inp{X_i}{w^*}X_iX_i^\top \bm{1} \lbrb{(X_i, \epsilon_i) \in \mc{H}}$ and consider the random matrix $W = n^{-1}\sum_{i = 1}^n W_i$. We first bound the mean of $W_i$. Take any unit vector $u$:
    \begin{align*}
        \mb{E} [u^\top W_i u] &= \mb{E} [\epsilon_i \inp{X_i}{w^*} \inp{X_i}{u}^2 (1 - \bm{1}\lbrb{(X_i, \epsilon_i) \notin \mc{H}})] \\
        &\leq (36 \nu \phi \psi \sigma^2 \sqrt{\eta} \log 1 / \eta) \mb{E} [\abs{\inp{X_i}{w^*}} \bm{1}\lbrb{(X_i, \epsilon_i) \notin \mc{H}}] \leq \eta^4 \sigma^2,
    \end{align*}
    where the last inequality follows from Cauchy-Schwarz, bounds on the moments of sub-Gaussian distributions and $\eta$ being sufficiently small. Proceeding identically to the previous term, we obtain by matrix-Bernstein that:
    \begin{equation*}
        \mb{P} \lprp{\norm{W} \geq \sigma^2 \eta} \leq \frac{1}{10d^2}.
    \end{equation*}
    Now, for any subset $S \subset [n]$ with $\abs{S} \geq (1 - 10\eta)n$, we have for any unit vector $u$:
    \begin{align*}
        n^{-1} \sum_{i \notin S} u^\top W_i u &= n^{-1} \sum_{i \notin S}  \epsilon_i \inp{X_i}{w^*}\inp{X_i}{u}^2 \bm{1} \lbrb{(X_i, \epsilon_i) \in \mc{H}} \\
        &\leq n^{-1} \sum_{i \notin S} (\epsilon_i^2\inp{X_i}{u}^2 + \inp{X_i}{w^*}^2\inp{X_i}{u}^2) \bm{1} \lbrb{(X_i, \epsilon_i) \in \mc{H}} \\
        &\leq O(\sigma^2 \eta \log^2 1 / \eta) + O(\sigma^2 \eta) = O(\sigma^2 \eta \log^2 1 / \eta),
    \end{align*}
    where we used \cref{lem:sgerrconcapp} for the first term and summed over all elements for the second term. By putting the previous results together, we have:
    \begin{equation*}
        \norm*{n^{-1} \sum_{i \in S} (\epsilon_i^2 X_i X_i^\top + 2\epsilon_i \inp{X_i}{w^*} X_i X_i^\top + \inp{X_i}{w^*}^2 X_iX_i^\top) \bm{1} \lbrb{(X_i, \epsilon_i) \in \mc{H}} - \sigma^2 \cdot I} \leq O(\eta \log^2 1 / \eta))\sigma^2,
    \end{equation*}
    where we used the \cref{lem:sgerrconcapp} to bound the deviation of the last term from $\sigma^2 \cdot I$ and the results proved in previous displays for the first and second term.
\end{proof}








We now present the main result of this section which establishes the deterministic conditions required for the success of \cref{alg:sgalg}. 

\begin{lemma}
    \label{lem:sgdetapp}
    Let $\mc{D}$ satisfy \cref{as:sgassump,as:sgiest}. Then, given $n = \tO (d / \eta^2)$ $2\eta$-corrupted samples from $\mc{D}$, we have with probability at least $1 - 1 / d^2$, there exists a set $S \subset [n]$ with $\abs{S} \geq (1 - 3\eta)n$ such that for any $T \subset S$ with $\abs{T} \geq (1 - 10\eta)n$:
    \begin{equation*}
        \norm{\mb{E}_T [Y_i X_i] - w^*} \leq O(1) \sigma \eta \log{1 / \eta} \text{ and } \norm{\mb{E}_T [Y_i^2 X_i X_i^\top] - \sigma^2 \cdot I} \leq O(1) \cdot \sigma^2 \eta \log^2 1 / \eta.
    \end{equation*}
    Furthermore, for all $i \in S$, we have $\norm{Y_i X_i} \leq O(1)\cdot \sigma \sqrt{\log 1 / \eta} (\sqrt{d} + \sqrt{\log 1 / \eta})$.
\end{lemma}
\begin{proof}[Proof of \cref{lem:sgdetapp}]
    Assume that the conclusions of \cref{lem:sgerrconcapp,lem:sgmeanconc,lem:sgcovconcapp}. This happens by the union bound with probability at least $1 - 9/(10d^2)$. Now, consider the set $S' = \{i: (X_i, \epsilon_i) \in \mc{H}\}$. Since, we have $\mb{P} ((X_i, \epsilon_i) \notin \mc{H}) \leq \eta^{16}$, we have via the Bernstein bound that with probability at least $1 - d^2/10$ that $\abs{S'} \geq (1 - \eta)n$. Now, since we have $2\eta n$ adversarial corruptions at least $(1 - 3\eta)n$ points of $S'$ are not corrupted. Let $S$ be the uncorrupted points of $S'$. This proves the final claim of the lemma.
    
    Let $T \subset S$ with $\abs{T} \geq (1 - 10\eta)n$. The second claim now follows from the conclusion of \cref{lem:sgcovconcapp}. The first term claim now follows from \cref{lem:sgmeanconc} and \cref{as:sgiest}.
    
    This concludes the proof of the lemma.
\end{proof}
\section{Lower Bound}
\label{sec:low_bd}

In this section, we prove that the recovery guarantees for our algorithms are optimal. In both cases, the lower bound  follows from similar arguments that for robust mean estimation. Concretely, we prove these lower bounds in the one dimensional case where $X, Y, w^*, \eps \in \mb{R}$ and in both cases, the distribution over $X$ is such that $X = 1$ with probability $1$. It is easy to see that this distribution over $X$ satisfies both \cref{as:htassump,as:sgassump}. In what follows, we use $d_{TV}$ to denote the total variation distance between two distributions. We first provide the lower bound for the heavy-tailed setting:

\begin{theorem}
    \label{thm:lb_ht}
    For any $\sigma > 0, \eta \in (0, 1)$, there exist two distributions, $\mc{D}_1$ and $\mc{D}_2$, over pairs $(X, Y)$ with $X, Y \in \mb{R}$, satisfying \cref{as:htassump} such that:
    \begin{equation*}
        d_{TV} (\mc{D}_1, \mc{D}_2) \leq \frac{\eta}{2} \text{ and } \abs{w^*_1 - w^*_2} \geq \Omega (\sigma \sqrt{\eta}),
    \end{equation*}
    \noindent where $w^*_1$ and $w^*_2$ are the true parameter vectors for $\mc{D}_1$ and $\mc{D}_2$ respectively.
\end{theorem}

\begin{proof}
    From \cref{as:htassump}, we are only required to specify the distributions over $X$ and $\eps$. In both cases, we use covariates $X$ such that $\mb{P} (X = 1) = 1$. It is clear that $X$ satisfies the L4-L2 hypercontractivity assumption of \cref{as:htassump}. Now, we design the noise distributions over $\eps$ for the two distributions. Using, $\eps_1$ and $\eps_2$ to denote the noise random variables for $\mc{D}_1$ and $\mc{D}_2$ respectively. We define the random variables, $\eps_1$ and $\eps_2$, as follows:
    \begin{equation*}
        \mb{P}(\eps_1 = 0) = 1 
        \text{ and } 
        \mb{P} (\eps_2 = x) = \begin{cases} 
                                \frac{\eta}{2}, &\text{if } x = \sigma \cdot \frac{1}{\sqrt{\eta}} \\
                                1 - \frac{\eta}{2}, &\text{if } x = \sigma \cdot \frac{-\sqrt{\eta}}{2(1 - \eta / 2)} \\
                                0, &\text{otherwise.}
                              \end{cases}
    \end{equation*}
    It can be checked that $\eps_1$ and $\eps_2$ satisfy \cref{as:htassump}. Finally, we define $w^*_1$ and $w^*_2$. Here, we simply set $w_1^* = 0$ and $w_2^* = \sigma \cdot \frac{\sqrt{\eta}}{2 (1 - \eta / 2)}$. Therefore, the TV distance between $\mc{D}_1$ and $\mc{D}_2$ is the same as the total variation distance between the distributions over the random variables $\eps_1$ and $\eps_2 + w_2^*$ which is at most $\frac{\eta}{2}$ from the definitions of $\eps_1$ and $\eps_2$. This concludes the proof of the theorem.
\end{proof}

The above proof immediately yields the following corollary which demonstrates that some dependence on the condition number is necessary in the reocvery of the parameter vector:

\begin{corollary}
\label{cor:condDep}
    Let $d \geq 2$ and $\kappa > 1$. Then, there exist two distributions, $\mc{D}_1$ and $\mc{D}_2$, over pairs $(X, Y)$ satisfying \cref{as:htassump} such that:
    \begin{equation*}
        d_{TV} (\mc{D}_1, \mc{D}_2) \leq \frac{\eta}{2} \text{ and } \abs{w^*_1 - w^*_2} \geq \Omega (\sigma \sqrt{\kappa \eta})
    \end{equation*}
    \noindent where $w_1^*$ and $w^*_2$ are true parameter vectors for $\mc{D}_1$ and $\mc{D}_2$ respectively.
\end{corollary}

\begin{proof}
    As in the proof of \cref{thm:lb_ht}, both $\mc{D}_1$ and $\mc{D}_2$ share the same distribution over $X$. Let $X = (X_1, \dots, X_d)$. Now, we will have have 
    \begin{equation*}
        \mb{P} (X_i = 1) = 1 \text{ for } i \in [d - 1] \text{ and }\mb{P}(X_d = 1 / \sqrt{\kappa}) = 1.
    \end{equation*}
    It is easy to see that $X$ satisfies \cref{as:htassump}. The distributions over $\eps_1$ and $\eps_2$ are identical to those in \cref{thm:lb_ht}. Finally, to determine $w_1^*$ and $w_2^*$, let $\wt{w}_1$ and $\wt{w}_2$ be the one dimensional parameter values from \cref{thm:lb_ht}. We now set $w_1^* = 0$ and $w_2^* = (0, \dots, 0, \wt{w}_2 \sqrt{\kappa})$. It can again be verified that $\mc{D}_1$ and $\mc{D}_2$ have total variation distance less than $\eta / 2$. From the specific settings of $w^*_1$ and $w^*_2$, the proof of the corollary follows.
\end{proof}

Through a similar technique, we establish a lower bound for the sub-Gaussian case.
\begin{theorem}
    \label{thm:lb_sg}
    For any $\sigma > 0, \eta \in (0, 1)$, there exist two distributions, $\mc{D}_1$ and $\mc{D}_2$, over pairs $(X, Y)$ with $X, Y \in \mb{R}$, satisfying \cref{as:sgassump} such that:
    \begin{equation*}
        d_{TV} (\mc{D}_1, \mc{D}_2) \leq \frac{\eta}{2} \text{ and } \abs{w^*_1 - w^*_2} \geq \Omega (\sigma \eta \sqrt{\log 1 / \eta}),
    \end{equation*}
    \noindent where $w^*_1$ and $w^*_2$ are the true parameter vectors for $\mc{D}_1$ and $\mc{D}_2$ respectively.
\end{theorem}
\begin{proof}
    As in the proof of \cref{thm:lb_ht}, we let $X$ be $1$ with probability $1$ for both $\mc{D}_1$ and $\mc{D}_2$. Again, $\mc{X}$ satisfies the sub-Gaussianity assumptions in \cref{as:sgassump} with $\psi = 1$. Deploying the notation from the proof of \cref{thm:lb_ht}, we now design the distribution over $\eps_1$ and $\eps_2$:
    \begin{equation*}
        \mb{P} (\eps_1 = 0) = 0 
        \text{ and } 
        \mb{P} (\eps_2 = x) = \begin{cases}
                                   \frac{\eta}{2}, & \text{if } x = \sigma \cdot \sqrt{\log 1 / \eta} \\
                                   1 - \frac{\eta}{2}, & \text{if } x = \sigma \cdot \frac{-\eta \sqrt{\log 1 / \eta}}{2(1 - \eta / 2)} \\
                                   0, &\text{otherwise.}
                              \end{cases}
    \end{equation*}
    Again, we see that both $\eps_1$ and $\eps_2$ satisfy the sub-Gaussianity assumptions in \cref{as:sgassump}. We now finally set $w_1^* = 0$ and $w_2^* = \sigma \cdot \frac{\eta \sqrt{\log 1 / \eta}}{2 (1 - \eta / 2)}$. As before, we have that the TV distance between $\mc{D}_1$ and $\mc{D}_2$ is at most $\frac{\eta}{2}$ which concludes the proof of the theorem.
\end{proof}

\section{Fast SDP Solvers}
\label{sec:fstsolv}

In this section, we prove the existence of nearly linear time solvers for the class of SDPs required in our algorithms. Our proof follows along the lines of \citep{cheng2019high} with slight reparametrizations convenient for our analysis. Recall that we aim to solve the following semidefinite program for a given set of points $\mc{Z} = \{Z_1, \dots, Z_n\}$:

\begin{equation}
    \label{eq:minmax} \tag{MT}
    \min_{s \in \psimp[\delta]}\lmax{\left(\sum_{i = 1}^n s_i Z_i Z_i^\top\right)}. 
\end{equation}
\cite{cheng2019high} solve this problem via a reduction to the following packing SDP by introducing an additional parameter $\lambda$:


\begin{equation}
\label{eq:packingred}
\begin{gathered}
    \max_{s} \sum_{i = 1}^n s_i \\
    \text{Subject to: } 0 \leq s_i \leq \frac{1}{(1 - \delta)n} \\
    \sum_{i = 1}^n s_i Z_i Z_i^\top \preccurlyeq \lambda \cdot I.
\end{gathered} \tag{Pack}
\end{equation}

\newcommand{\optv}[1][\delta]{\text{OPT}^*_{#1}}
\newcommand{\sumlow}{l^*}
\newcommand{\optp}{\text{Pack}^*_{\delta, \lambda}}

Let $\optv{}$ denote the optimal value of the program \ref{eq:minmax}, \ref{eq:minmax}$(\delta)$ denote the program instantiated with $\delta$ and let \ref{eq:packingred}$(\delta, \lambda)$ denote the program $\text{\ref{eq:packingred}}$ instantiated with $\delta, \lambda$ and let $\optp$ denote its optimal value. 
The following quantity is useful throughout the section:
\begin{equation*}
    \sumlow = \min_{s \in \psimp[\delta]} \sum_{i = 1}^n s_i \norm{Z_i}^2.
\end{equation*}
This is simply the average squared lengths of the $(1 - \delta) n$ smallest $Z_i$. We introduce a technical result useful in the following analysis:
\begin{lemma}
    \label{lem:optval}
    \ref{eq:packingred}$(\delta, \optv)$ has optimal value at least $1$. 
\end{lemma}
\begin{proof}
    The lemma follows from the fact that a feasible solution for \ref{eq:minmax}$(\delta)$ is a feasible solution for \ref{eq:packingred}$(\delta, \lambda)$ for $\lambda \geq \optv$.
\end{proof}
The following lemma proves that $\sumlow$ gives an approximation to $\optv$ within a factor of $d$.

\begin{lemma}
    \label{lem:sumlowapx}
    The value $l^*$ satisfies:
    \begin{equation*}
        \optv \leq l^* \leq d\optv.
    \end{equation*}
\end{lemma}
\begin{proof}
    The upper bound on $\optv$ follows from that fact that $\lmax{Z_iZ_i^\top} \leq \norm{Z_i}^2$ and the lower bound follows from the inequality $\tr M \leq d \lmax{M}$ for any psd matrix $M$.
\end{proof}

\newcommand{\optlamb}[1][\lambda]{\text{OPT}_{#1}}

In what follows we prove that we can efficiently binary search over the value of $\lambda$ to find a good solution to \ref{eq:minmax}. We  refer to $\optlamb{}$ as the optimal value of \ref{eq:packingred} run with $\lambda$. We now show a lemma analogous to Lemma~4.3 in \citep{cheng2019high} relating to the monotonicity of properties of $\optlamb{}$ viewed as a function of $\lambda$. 

\begin{lemma}
    \label{lem:contl}
    The function, $\optlamb$ when viewed as a function of $\lambda$ is monotonic in $\lambda$.
\end{lemma}
\begin{proof}
    The lemma follows from the observation that for $\lambda_1 \geq \lambda_2$, a feasible point for \ref{eq:minmax} with $\lambda_2$ is a feasible point for the program with $\lambda_1$. 
\end{proof}

We now restate a lemma from \citep{cheng2019high}.

\begin{lemma}
    \label{lem:packtomm}
    Given a feasible point for \ref{eq:packingred}$(\delta, \lambda)$, $s_i$, with $\sum_{i = 1}^n s_i \geq 1 - \delta / 10$, one can find a feasible solution for \ref{eq:minmax}$(2\delta)$ with objective value at most $(1 + \delta / 2) \lambda$.
\end{lemma}
\begin{proof}
    The lemma follows by considering the solution $s'_i = s_i / (\sum_{i = 1}^n s_i)$. $s' \in \psimp[2\delta]$ as we have $s'_i \leq \frac{1}{(1 - \delta) (1 - \delta / 10)n} \leq \frac{1}{(1 - 2\delta)n}$. The bound on the objective value follows from:
    \begin{equation*}
        \sum_{i = 1}^n s'_i Z_i Z_i^\top \preccurlyeq (1 - \delta / 10)^{-1} \sum_{i = 1}^n s_i Z_i Z_i^\top \preccurlyeq (1 + \delta / 2) \lambda I. 
    \end{equation*}
\end{proof}

We now present a Corollary~4.5 from \citep{cheng2019high}, a consequence of the existence of fast solvers for packing sdps from \citep{peng2012faster}:

\begin{corollary}
    \label{lem:fastpack}
    Fix $0 \leq \delta \leq 1/3$ and $\lambda > 0$. Then, one can obtain in time $\tO(nd / \delta^6)$, with probability at least $9 / 10$, a feasible point for \ref{eq:packingred}$(\delta,\lambda )$, $\hat{s}$, with:
    \begin{equation*}
        \sum_{i = 1}^n \hat{s}_i \geq (1 - \delta / 10) \optp.
    \end{equation*}
\end{corollary}

We conclude with the main lemma of the section. 
\begin{lemma}
    \label{lem:fastsdp}
    Given $\mc{Z} = \{Z_i\}_{i = 1}^n$ and $\delta$, one can find with probability at least $9 / 10$, a feasible solution to \ref{eq:minmax}$(2\delta)$ with objective value at most $(1 + \delta) \optv[\delta]$.
\end{lemma}
\begin{proof}
    We use a binary search procedure to find an appropriate value of $\lambda$. To do this, we maintain two indices $\lambda_l$ and $\lambda_h$ satisfying the following two conditions:
    \begin{enumerate}
        \item $\optv \geq \lambda_l$ and
        \item We have a feasible point $s^h$ for \ref{eq:packingred}$(\delta, \lambda_h)$ with $\sum_{i = 1}^h s^h_i \geq (1 - \delta / 10)$.
    \end{enumerate}
    At the beginning we instantiate $\lambda_l = l^* / d$ and $\lambda_h = l^*$.  The upper bound on $l^*$ in \cref{lem:sumlowapx} ensures the first condition. For the second one, since $l^* \geq \optv$, we get from \cref{lem:optval,lem:contl,lem:fastpack} the Algorithm from \citep{peng2012faster} obtains a feasible solution for \ref{eq:packingred}$(\delta, \lambda_h)$ with objective value at least $(1 - \delta / 10)$. We then run \ref{eq:packingred} with $\lambda = \lambda_{m} = (\lambda_l + \lambda_h) / 2$ and set $\lambda_h = \lambda_{m}$ if the objective value is greater than $1 - \delta / 10$. Otherwise, we set $\lambda_l = \lambda_m$. The second condition is trivially maintained while the first condition follows from \cref{lem:optval,lem:fastpack,lem:contl}. We run the binary search for $O(\log d / \delta)$ rounds so that we have at the end of the procedure $\lambda_h - \lambda_l \leq \frac{\delta}{10} \optv$ as we have $\lambda_h - \lambda_l \leq d \optv$ at the beginning from Lemma~\ref{lem:sumlowapx}. At the end, we return the feasible solution for $\lambda_h$. The approximation guarantee follows from the fact that $\lambda_h \leq \lambda_l + \frac{\delta}{10} \optv \leq (1 + \frac{\delta}{10})\optv$ and \cref{lem:packtomm}. As in \citep{cheng2019high}, we run the Algorithm $O(\log \log d / \delta)$ in each step to ensure probability $9/10$ over all rounds. 
\end{proof}

\section{Linear Algebra and Probability Results}

Here we collect the statements (and proofs) of useful results from linear algebra and probability. 

\begin{lemma}\label[lemma]{lem:linalgmean}
Let $M$ be a (random) symmetric matrix, then $\normt{\mb{E}[M]^2} \leq \normt{\mb{E}[M^2]}$ and $\normt{\mb{E}(M-\mb{E}[M])^2} \leq \normt{\mb{E}[M^2]}$.
\end{lemma}
\begin{proof}[Proof of \cref{lem:linalgmean}]
    Note that $\mb{E}[(M-\E[M])^2] \succeq 0 \implies \mb{E}[M^2] \succeq \mb{E}[M]^2$. Since both matrices are p.s.d. it follows that $\normt{\mb{E}[M]^2} \leq \normt{\mb{E}[M^2]}$. The second claim follows from the variational characterization of the operator norm of the p.s.d. matrix $(M-\mb{E}[M])^2$.
\end{proof}

\begin{lemma} \label[lemma]{lem:l4l2norm}
    Let $X \sim \cD$ be a random vector from a distribution that is L$4$-L$2$ hypercontractive -- $\mb{E}[\langle v, X \rangle^4] \leq L^2 (\mb{E}[\langle v, X \rangle^2])^4$ -- then
    \[\normt {\mb{E}[\normt{X}^2 XX^\top]} \leq L \Tr(\Sigma) \normt{\Sigma}. \]
\end{lemma}
\begin{proof}[Proof of \cref{lem:l4l2norm}]
     We introduce a vector $v$ with $\normt{v} \leq 1$. Then,
\begin{align*}
    & \mb{E}[\iprod{v, \normt{X}^2 XX^\top v}] = \mb{E}[\normt{X}^2 \iprod{v, X}^2] \leq (\mb{E}[\normt{X}^4])^{1/2} (\mb{E}[\iprod{v, X}^4])^{1/2},
\end{align*}
by Cauchy-Schwarz and the Jensen inequality.
For the first term we have $(\mb{E}[\normt{X}^4])^{1/2} \leq \sqrt{L} \Tr \Sigma$ by \cref{lem:l8l2vec}. For the second term once again using L$4$-L$2$ hypercontractivity we have,
$(\mb{E}[\iprod{v, X}^4])^{1/2} \leq \sqrt{L} \mb{E}[\iprod{v, X}]^2 \leq \sqrt{L} \normt{\Sigma}$. Maximizing over $v$ gives the result.
\end{proof}

\begin{lemma} \label[lemma]{lem:l8l2vec}
    Let $X \sim \cD$ be a random vector with a distribution that is L$4$-L$2$ hypercontractive. Then,
    \begin{align*}
      \mb{E}[\normt{X}^4] \leq L (\Tr \Sigma)^2.
    \end{align*}
\end{lemma}

\begin{proof}[Proof of \cref{lem:l8l2vec}]
A short computation using the Cauchy-Schwarz inequality and L4-L2 equivalence shows that, 
\begin{align*}
     \mb{E}[\normt{X}^4] = \mb{E}[(\sum_{i=1}^d \iprod{X, e_i}^2)^2] =& \mb{E}[\sum_{a,b} \iprod{X, e_a}^2 \iprod{X, e_b}^2]  \\
    \leq & \sum_{a,b} (\mb{E}[\iprod{X, e_a}^4] \mb{E}[\iprod{X, e_b}^4])^{1/2}  \\
     \leq &  L \sum_{a,b} \mb{E}[\iprod{X, e_a}^2] \mb{E}[\iprod{X, e_b}^2] \leq L (\Tr \Sigma)^2.
\end{align*}
\end{proof}

\begin{lemma}
    \label{lem:lensg}
    Let $X$ be a sub-Gaussian random vector with sub-Gaussian parameter $\phi \geq 1$ and second moment $I$. Then, we have:
    \begin{equation*}
        \mb{P} \lprp{\norm{X} \geq 7\phi \sqrt{d} + 6 \phi \sqrt{\log 1 / \delta}} \leq \delta.
    \end{equation*}
\end{lemma}
\begin{proof}[Proof of \cref{lem:lensg}]
    Let $\mu$ denote the mean of $X$. We first note that we have $\norm{\mu} \leq \mb{E}[\norm{X}] \leq \sqrt{d}$. Therefore, the mean of $X$ is less than $\sqrt{d}$. Now, pick an $1/2$-net over the unit sphere, $\mc{G}$. By standard bounds, we have that the number of elements in $\mc{G}$ can be upper bounded by $6^d$ \cite{vershynin2012introduction}. Now, for any $u \in \mc{G}$, we have:
    \begin{equation*}
        \mb{P} [\abs{\inp{u}{X - \mu}} \geq t] \leq 2 \exp \lprp{- \frac{t^2}{2\phi^2}}.
    \end{equation*}
    Therefore, $t = 3\phi \sqrt{d + \log 1 / \delta}$, and taking a union bound over the at most $O(6^d)$ elements in $\mc{G}$, we have with probability $\delta$:
    \begin{equation*}
        \norm{X} \leq \norm{\mu} + 2 \max_{u \in \mc{G}} \inp{X}{u} \leq 7 \phi \sqrt{d} + 6 \phi \sqrt{\log 1 / \delta}.
    \end{equation*}
\end{proof}

\begin{lemma}
    \label{lem:sgtrunc}
    Let $X$ be a sub-Gaussian random variable with second moment $1$ and sub-Gaussianity parameter $\phi$. Then, we have for sufficiently small $\eta$:
    \begin{equation*}
        \mb{E} [X^4 \bm{1} \lbrb{\abs{X} \geq 6\phi \sqrt{\log 1 / \eta}}] \leq \phi^4 \eta^8.
    \end{equation*}
\end{lemma}
\begin{proof}[Proof of \cref{lem:sgtrunc}]
    Since we have $\abs{\mb{E} [X]} \leq \mb{E} [\abs{X}] \leq 1$. Now, we get:
    \begin{equation*}
        \mb{E} [X^4 \bm{1} \lbrb{\abs{X} \geq 6\phi \sqrt{\log 1 / \eta}}] \leq (\mb{E} [X^8])^{1/2} (\mb{P}[\abs{X} \geq 6 \phi \sqrt{\log 1 / \eta}])^{1/2} \leq \phi^4 \eta^8.
    \end{equation*}
\end{proof}

\begin{lemma}
    \label{lem:sgmoms}
    Assume the setting of \cref{lem:sgtrunc}. Then, we have for all $k$:
    \begin{equation*}
        \mb{E} [X^k] \leq k^{k / 2} (C \phi)^k,
    \end{equation*}
    for some absolute constant $C$.
\end{lemma}
\begin{proof}[Proof of \cref{lem:sgmoms}]
    The result follows from applying the result of Lemma 5.5 to the centered sub-Gaussian random variable $Y = X - \mb{E} [X]$ and from the fact that $\mb{E} [(Y + \mb{E}[X])^k] \leq 2^k \mb{E}[Y^k + \mb{E}^k[X]]$ and noting that $\mb{E}[\abs{X}] \leq 1$ as in the proof of \cref{lem:sgtrunc}.
\end{proof}

\begin{lemma}
    \label{lem:extpts}
    Let $n$ be an integer and let $\Delta_{\delta, n}$ be the set of distributions over $[n]$ with $0 \leq s_i \leq \frac{1}{(1 - \delta)n}$ for all $s \in \Delta_{\delta}$. Then, the points $\mc{E}_\delta = \{\{\alpha_i\}_{i = 1}^n: \alpha_i = 1 / ((1 - \delta)n) \text{ for $(1 - \delta)n$ elements}\}$ form the extreme points of $\Delta_{\delta}$.
\end{lemma}
\begin{proof}[Proof of \cref{lem:extpts}]
    Note that $\Delta_{\delta}$ is a polytope in $n$ dimension and all points in $\mc{E}_\delta$ satisfy at least $d$ linearly independent constraints of the polytope. Therefore, all the points in $\mc{E}_\delta$ are extreme points. To show that $\mc{E}_\delta$ contains all extreme points, consider an extreme point $\alpha$ not in $\mc{E}_\delta$. Then, there exists $i$ such that $\alpha_i \neq 0, 1 / ((1 - \delta)n)$. Since $\delta n$ is an integer, there exists $j \neq i$ such that $\alpha_j \neq 0, 1 / ((1 - \delta)n)$. Therefore, for small  $\xi$, we have $\alpha'$ defined by $\alpha'_k = \alpha_i$ for all $k \neq i, j$ and $\alpha'_i = \alpha_i + \xi$ and $\alpha'_j = \alpha_j - \xi$ is in $\Delta_{\delta, n}$ which contradicts $\alpha$ being an extreme point. Therefore, $\mc{E}_\delta$ contains all extreme points.
\end{proof}

\begin{lemma}
    \label{lem:weighttv}
    Assume the setting of \cref{lem:extpts} for $\delta_1$ and $\delta_2$. Then, we have for any $\alpha \in \Delta_{\delta_1, n}$ and $\beta \in \Delta_{\delta_2, n}$:
    \begin{equation*}
        \text{Dist}_{TV} (\alpha_i, \beta_i) \leq \delta_1 + \delta_2.
    \end{equation*}
\end{lemma}
\begin{proof}[Proof of \cref{lem:weighttv}]
    We prove that $\text{Dist}_{TV} (\alpha_i, \text{Unif} ([n])) \leq \delta_1$. This follows from the fact that maximizing the total variation distance from $\text{Unif} ([n])$ is a linear program over $\Delta_{\delta, n}$ and it is maximized at the extreme points described by \cref{lem:extpts}. The result follows from a similar proof for $\beta$ and triangle inequality.
\end{proof}



\end{document}